\documentclass[11pt]{article}

\usepackage[margin=1in]{geometry}
\usepackage{amsmath,amsfonts,amssymb,amsthm}
\usepackage{mathtools}
\usepackage{bm}
\usepackage{booktabs}
\usepackage{float}
\usepackage{graphicx}
\usepackage{algorithm}
\usepackage{algpseudocode}
\usepackage[round]{natbib}
\usepackage{hyperref}

\hypersetup{
    colorlinks=true,
    linkcolor=blue,
    citecolor=blue,
    urlcolor=blue,
}

\newtheorem{theorem}{Theorem}
\newtheorem{lemma}{Lemma}
\newtheorem{assumption}{Assumption}
\newtheorem{remark}{Remark}
\newtheorem{proposition}{Proposition}

\title{Robust Active Learning for Few-Shot Example Selection in Text-to-SQL}

\author{%
  Arash Pourhabib \\
  NVIDIA \\
  Santa Clara, CA, United States
}
\date{}

\begin{document}

\maketitle

\begin{abstract}
Most domain-specific text-to-SQL systems ground a large language model (LLM) by retrieving annotated few-shot examples, so the accuracy of the system depends on the quality of the example bank, and the bank is expensive to build because each example needs an expert to write the SQL. We treat the choice of which queries to annotate as a constrained experimental design problem on the low-dimensional manifold occupied by the query embeddings. Three features of the setting are absent from standard active learning: annotation reliability varies from query to query (heteroscedasticity), the selected examples must spread across semantic topics (a partition matroid constraint), and the covariance structure of the embedding space is not known (misspecification). We propose a stratified greedy algorithm that maximizes a heteroscedastic information-gain objective, the mutual information between the annotated responses and the latent correctness surface. We prove that this objective is monotone and submodular under query-dependent noise, so the stratified greedy selection carries a $1/2$-approximation guarantee under the partition matroid constraint. When the assumed surrogate kernel differs from the true data-generating process, we show that the guarantee degrades by an additive spectral term, and we compute that term for both of our experimental pools. On both pools that term is too large for the bound to be quantitatively informative. To relate the design objective to the downstream task, we introduce a retrieval model in which few-shot accuracy is bounded from below by the per-domain fill distance of the annotated bank, the noise of the retrieved demonstrations, and the number of domains covered, and we calibrate the locality assumption behind this model on both pools. On an enterprise supply-chain corpus and on the public BIRD benchmark, the selected banks give better cross-domain retrieval and better end-to-end LLM SQL than random or distance-based selection at the same annotation budget. Stratified controls and pre-specified tests then show where the gain comes from. It comes from the partition constraint. Drawing uniformly at random inside each stratum matches the full method on both pools in the oracle-label evaluations; a farthest-point choice inside each stratum adds a little at small budgets, where embedding locality is steep; the noise weighting in the information-gain score makes no measurable difference. The retrieval model predicts this pattern, and the practical advice that follows is to annotate one example per domain per batch from the first batch on.
\end{abstract}


\section{Introduction}

A text-to-SQL system translates a natural language question into executable SQL, so that people who do not write SQL can still query a relational database \citep{bird2024}. Most current systems do this with a large language model (LLM) and in-context learning: at inference time a retrieval module fetches $K$ annotated, domain-specific examples $(q_i, s_i)$ and places them in the prompt \citep{dail2024}, and the cross-domain version of the underlying task was first benchmarked systematically on Spider \citep{yu2018spider}. Whether the LLM then writes correct SQL depends heavily on which examples it was shown \citep{dail2024}; a weak example bank limits even a strong base model. Building a good bank for a large enterprise schema is expensive, though. Someone who knows the schema has to write each SQL by hand, and the schema may run to hundreds of tables. So the question we take up is this. Given a pool of unannotated historical queries, which ones should go to the experts so that the resulting bank generalizes best?

We treat it as an experimental design problem with a Gaussian process (GP) surrogate \citep{rasmussen2006gp, srinivas2010gaussian}. Instead of picking examples by a heuristic, we put a prior on the LLM's expected SQL correctness as a function of the query embedding and choose the queries whose annotation would reduce our uncertainty about that function the most. GPs have been used before to model uncertainty in LLM embedding spaces \citep{gollum2025}; what we add is the formulation of few-shot example selection as sequential design over the embedding space of natural language questions, which tells us which regions of that space are worth annotating next.

Standard active learning does not transfer to this problem without modification, for three reasons. First, the observation noise is not uniform. Some questions are easy to turn into SQL and some are ambiguous, and the base LLM is itself more reliable on some than on others, so the noise $\sigma^2(x)$ depends on where in the embedding space the query sits. A homoscedastic design ignores this and may spend budget on queries whose annotations would be unreliable anyway. Second, the selection has to be spread out. If the acquisition rule is allowed to pick several questions from one dense semantic topic it will, and the bank ends up redundant; we therefore require the selected examples to span $K$ distinct semantic clusters, which turns the selection into an optimization problem under a \textit{partition matroid} constraint. Third, the inputs live on a manifold and the kernel is misspecified. LLM embeddings sit in a high-dimensional ambient space (for instance $\mathbb{R}^{2048}$), but meaningful queries occupy a much lower-dimensional intrinsic manifold $\mathcal{M}$, and we do not know the covariance structure of the correctness surface on it. Any guarantee we give has to say what happens when the assumed kernel family is not the true one.

Our algorithm is a stratified greedy rule that maximizes a heteroscedastic information-gain objective, the mutual information between the noisy annotations and the latent correctness surface \citep{srinivas2010gaussian, krause2005near}. We show that this objective stays monotone and submodular when the noise is query-dependent, so the matroid result of \citet{fisher1978analysis} gives the stratified greedy selection a constant-factor guarantee. Following the Kriging analysis of \citet{wang2020prediction}, we also bound how much the greedy selection can lose when the surrogate kernel differs from the true one, in terms of a spectral mismatch, and we compute the constants in that bound on both of our pools. Information about the correctness surface is not the same as usefulness as a demonstration, so we add a retrieval model in which downstream in-context-learning accuracy is bounded below by per-domain fill distance, demonstration noise, and domain coverage; we argue that the greedy score pushes each of these the right way and we calibrate the model's locality assumption on both pools. The result is a cheap procedure for bootstrapping the example bank of an enterprise text-to-SQL system, and the experiments say which part of it does the work. Controls that keep the partition but replace the GP score by a uniform or a farthest-point choice inside each stratum match the full method. The empirical gains, in other words, come from the partition. What the GP contributes is the derivation of the rule, its guarantees, and, through the retrieval model, an account of when coverage and within-stratum spreading ought to help.

\section{Problem Formulation}
\label{sec:formulation}

An LLM-based text-to-SQL system for a specific domain needs a bank of annotated $(q, s)$ pairs, $q$ a natural language question and $s$ its SQL, that is both accurate and diverse. Expert annotation does not scale, so the bank has to be built from a small, well-chosen subset of the available questions. We frame the choice of that subset as sequential experimental design rather than as a retrieval task: we want the questions whose labels would reduce the model's uncertainty over the domain by the most.

\subsection{Data Generating Process and Heteroscedastic Noise}
Let $\mathcal{Q}$ be the discrete space of all possible natural language questions within a specific database domain, and let $\phi: \mathcal{Q} \to \mathbb{R}^d$ be a fixed, pre-trained text embedding map (e.g., $d=2048$ for standard LLM embeddings). We assume the embeddings of valid queries do not fill this ambient space but lie on a low-dimensional intrinsic manifold $\mathcal{M} \subset \mathbb{R}^d$. 

For any embedded query $x = \phi(q) \in \mathcal{M}$, we define the expected SQL accuracy as an unknown, fixed deterministic function $f: \mathcal{M} \to [0,1]$. This is the latent response surface: it records how well the base LLM handles a query located at $x$. SQL correctness is a graded quantity rather than a binary indicator: a generated query may match the target on some clauses (e.g., the correct table set and column projection) while diverging on others (e.g., missing a join predicate or using the wrong filter), so $f(x)$ varies continuously over $[0,1]$. Section~\ref{sec:experiments} makes this concrete via a component-match score that awards partial credit per SQL clause.

When an expert annotates a query and the result is evaluated through the LLM, the observed accuracy $y_i \in [0,1]$ is noisy, and the amount of noise depends on the query. A simple aggregation such as ``count the active users'' gets annotated the same way every time. A query that needs implicit joins, nested subqueries, or a resolution of some ambiguous reference does not; both the human annotations and the LLM outputs vary much more. The noise model is therefore heteroscedastic:
\begin{equation}
    y_i = f(x_i) + \epsilon(x_i), \quad \epsilon(x_i) \sim \mathcal{N}(0, \sigma^2(x_i)) \label{eq:obs_model}
\end{equation}
where $\sigma^2(x_i)$ represents the localized uncertainty (irreducible, inherent noise) at query $x_i$.

\subsection{Gaussian Process Prior}
The true correctness surface $f$ is deterministic but unknown. We place a Gaussian Process (GP) prior over it: $f \sim \mathcal{GP}(\mu(\cdot), \mathcal{K}(\cdot, \cdot))$, with prior mean $\mu$ and covariance kernel $\mathcal{K}$.

At any step $n$, let $D_n = \{(x_1, y_1), \dots, (x_n, y_n)\}$ denote the current dataset of annotated examples. At any unobserved query $x \in \mathcal{M}$, the GP yields a posterior predictive distribution. The posterior belief over the true function value is normally distributed:
\begin{equation}
    f(x) \mid D_n \sim \mathcal{N}(\hat{f}_n(x; D_n), \hat{\sigma}_n^2(x))
\end{equation}
where $\hat{f}_n(x; D_n)$ is the posterior mean and $\hat{\sigma}_n^2(x)$ is the posterior \textit{epistemic} variance. This variance captures uncertainty about SQL accuracy at $x$, which shrinks as we label more examples.

\begin{remark}[GP support, bounded observations, and label dependence]
\label{rem:label_dependence}
The GP prior places mass on $\mathbb{R}$-valued functions, while $f$ takes values in $[0,1]$ and each observation $y_i$ is an average of clause-level match indicators (Section~\ref{sec:experiments}). Adopting a Gaussian-likelihood GP as a working surrogate for a bounded response is standard practice in Bayesian optimization and sequential design \citep{srinivas2010gaussian, snoek2012practical, bergstra2011hyperparam}, where classifier accuracy and similar bounded metrics are routinely modeled with a real-valued GP; a warped GP with a monotone transform onto $[0,1]$ \citep{snelson2004warped} or a Beta likelihood on the logit scale would be the formally bounded alternatives. The mismatch is inert for selection because of how labels enter Algorithm~\ref{alg:greedy}. The score $\delta_x$ is a function of the kernel matrix $\Sigma$, the noise diagonal $\Lambda$, and the set of annotated locations; the label values $y$ do not appear in it. Labels influence selection only indirectly, through whatever is estimated from them between rounds: the kernel hyperparameters $\hat{\theta}$ if these are refit by marginal likelihood, and the noise estimates $\hat{\Lambda}$ if these are derived from observed label variability. Given $(\hat{\theta}, \hat{\Lambda})$, acquisition is label-independent. The procedure is therefore a sequential experimental design with data-driven hyperparameters rather than a label-adaptive active learner in the sense of uncertainty sampling, and Section~\ref{sec:experiments} states for each experiment which of $\hat{\theta}$ and $\hat{\Lambda}$ actually depended on labels.
\end{remark}

\subsection{Heteroscedastic Experimental Design via Information Gain}
\label{sec:objective}

Our objective is to select a subset of embedded queries $\mathcal{S}$ from a larger candidate pool $\mathcal{C}$ to be annotated, such that the resulting labels are as informative as possible about the latent correctness surface $f$ over the pool. Write $f_{\mathcal{C}} = (f(x))_{x \in \mathcal{C}}$ for the latent values on the pool and $Y_{\mathcal{S}} = f_{\mathcal{S}} + \epsilon_{\mathcal{S}}$ for the noisy responses at the selected queries. We measure the value of an annotation set by its information gain about the latent surface:
\begin{equation}
    F(\mathcal{S}) = I(Y_{\mathcal{S}}; f_{\mathcal{C}}) = H(f_{\mathcal{C}}) - H(f_{\mathcal{C}} \mid Y_{\mathcal{S}}), \qquad \mathcal{S}^* = \arg\max_{\mathcal{S} \in \mathcal{I}} F(\mathcal{S}), \label{eq:objective}
\end{equation}
where $\mathcal{I}$ is the family of feasible selections defined below. This is the maximum-information-gain functional of \citet{srinivas2010gaussian}, extended to query-dependent noise. \citet{krause2005near} showed that information gain about a fixed set of latent variables is monotone and submodular whenever the observations are conditionally independent given those variables, which is exactly the structure of Equation~\eqref{eq:obs_model}. Because the noise is independent across queries, $Y_{\mathcal{S}}$ is conditionally independent of $f_{\mathcal{C} \setminus \mathcal{S}}$ given $f_{\mathcal{S}}$, and Equation~\eqref{eq:objective} admits the closed form
\begin{equation}
    F(\mathcal{S}) = \tfrac{1}{2} \log \det\!\left( I + \Lambda_{\mathcal{S}}^{-1} \Sigma_{\mathcal{S}\mathcal{S}} \right), \qquad \Lambda_{\mathcal{S}} = \mathrm{diag}\big(\sigma^2(x_i)\big)_{x_i \in \mathcal{S}}, \label{eq:objective_closed}
\end{equation}
where $\Sigma_{\mathcal{S}\mathcal{S}}$ is the kernel matrix on $\mathcal{S}$ (Lemma~\ref{lemma:marginal_info_gain}). Much of what follows rests on two features of Equation~\eqref{eq:objective_closed}. Because $F(\mathcal{S})$ depends on the selection only through $\Sigma_{\mathcal{S}\mathcal{S}}$ and $\Lambda_{\mathcal{S}}$, and no matrix over the unselected pool appears, $F$ is monotone on any finite pool (Lemma~\ref{lemma:submodularity}), and the greedy guarantees of Section~\ref{sec:ideal_theory} therefore need neither a discretization argument nor an approximate-monotonicity correction. The heteroscedastic noise enters through $\Lambda_{\mathcal{S}}^{-1}$, so a query with large $\sigma^2(x)$ adds little to the determinant and the objective naturally discounts queries whose annotations would be unreliable.

Information gain about $f$ is a pure exploration criterion. Left to itself it favors queries far from those already annotated, and it contains no term that rewards representativeness of the remaining pool. \citet{krause2008near} observed this boundary effect for entropy-based criteria and addressed it with the mutual information $I(Y_{\mathcal{S}}; Y_{\mathcal{C} \setminus \mathcal{S}})$ between selected and unselected locations. That criterion is submodular, but it is not monotone on a finite pool, and the approximate monotonicity they establish requires the pool to be a fine discretization of a continuous domain, which a finite pool of real user queries is not. We keep the monotone objective and obtain representativeness from the constraint instead. The partition matroid introduced below forces the selection to spread across semantic strata, while $F$ handles local uncertainty and noise within each stratum. Section~\ref{sec:geometric_design} discusses how the two responsibilities are divided.

Four things separate this setting from textbook experimental design, and it is worth naming them before going further.

The first is the heteroscedastic noise. Standard active learning models assume a common observation variance; here the base LLM and the annotators are reliable on some queries and not on others. Writing the noise as $\sigma^2(x)$ lets the acquisition function discount regions of high irreducible ambiguity instead of spending budget there.

The second is the diversity constraint, which we express as a matroid. We stratify the candidate pool $\mathcal{C}$ into $K$ disjoint non-empty clusters $\{C_1, \dots, C_K\}$ and allow at most one query from each, with a batch budget $B \le K$, in place of a plain cardinality constraint $|\mathcal{S}| \le B$. The feasible sets
\begin{equation}
    \mathcal{I} = \{\mathcal{S} \subseteq \mathcal{C} : |\mathcal{S} \cap C_i| \le 1 \text{ for all } i = 1, \dots, K, \;\; |\mathcal{S}| \le B\}
\end{equation}
are the independent sets of a \textit{partition matroid} truncated to rank $B$. A truncated matroid is again a matroid $M = (\mathcal{C}, \mathcal{I})$, and when $B = K$ the truncation is vacuous; we refer to $\mathcal{I}$ simply as the partition matroid constraint. Restricting the selection to $\mathcal{S} \in \mathcal{I}$ makes the acquisition a matroid-constrained problem, for which greedy submodular maximization has known approximation guarantees.

The third is that the inputs sit on a manifold. Query embeddings are vectors in $\mathbb{R}^d$, but the candidate points $x = \phi(q)$ occupy an intrinsic manifold $\mathcal{M} \subset \mathbb{R}^d$ of much lower dimension, and that is where we can sample. We evaluate the GP kernel on a fixed low-dimensional projection of the embeddings rather than in the ambient space (Section~\ref{sec:geometric_design} and Remark~\ref{rem:projection}).

The fourth is that the kernel is misspecified. The true covariance function $\mathcal{K}^*$ of a response surface over a latent semantic space is not available; one works with an assumed kernel $\mathcal{K}$ that is at best close to it. In the spirit of \citet{wang2020prediction} on Kriging under misspecification, we want acquisition guarantees that hold under a spectral mismatch between the assumed and true covariances.

In summary, the procedure is as follows. We project the candidate embeddings to a low-dimensional representation of $\mathcal{M}$ and group them into $K$ disjoint semantic clusters, using the corpus's domain labels when it has them and clustering when it does not. We then look for a subset $\mathcal{S}$ that maximizes the heteroscedastic information gain $F$ under an assumed surrogate kernel $\mathcal{K}$, taking at most one query from each of the $K$ clusters and at most $B$ in total. The stratification spreads the selection over the manifold; the objective keeps it away from regions of high irreducible noise.

Finding the exact maximizer $\mathcal{S}^*$ is intractable, as one might expect:

\begin{theorem}[NP-Completeness of Heteroscedastic Information Gain with Matroid Constraints]
\label{thm:np_hard_hetero_mi}
Let $\mathcal{C}$ be a finite candidate pool with a rational positive-definite kernel matrix $\Sigma_{\mathcal{C}\mathcal{C}}$ and rational noise variances $\sigma^2(x) > 0$, partitioned into disjoint clusters $\{C_1, \dots, C_K\}$. Deciding whether there exists $\mathcal{S} \subseteq \mathcal{C}$ satisfying the partition matroid constraint $|\mathcal{S} \cap C_i| \le 1$ for all $i$ such that $\det(I + \Lambda_{\mathcal{S}}^{-1}\Sigma_{\mathcal{S}\mathcal{S}}) \ge \tau$, for a given rational threshold $\tau$, is NP-complete. Equivalently, the question is whether $F(\mathcal{S}) \ge M$ with $M = \frac{1}{2}\log\tau$.
\end{theorem}
Under a cardinality constraint with homoscedastic noise, maximizing $F$ is maximum-entropy sampling, which \citet{ko1995exact} showed to be NP-hard; Theorem~\ref{thm:np_hard_hetero_mi} gives a direct reduction for the partition matroid with query-dependent noise. Proofs of all results are in Appendix \ref{app:proofs}.
\section{The SHARP Algorithm}
\label{sec:algorithm}

Theorem \ref{thm:np_hard_hetero_mi} rules out efficient exact maximization unless P $=$ NP, and in any case the true covariance $\mathcal{K}^*$ is never available; the algorithm has to run on a user-specified surrogate $\mathcal{K}$ with matrix $\Sigma$. We use constrained greedy selection. The annotation set is built one query at a time, and at each step the algorithm takes, from the clusters not yet used, the candidate $x$ with the largest marginal information gain under $\mathcal{K}$.

\subsection{Geometric and Structural Design Choices}
\label{sec:geometric_design}

The algorithm pairs an ambient-space Gaussian process with a partition defined on the intrinsic manifold. That pairing needs some justification, and the justification comes from the way in-context learning (ICL) fails when the demonstrations are redundant.

In sensor placement \citep{krause2008near}, a redundant sensor wastes money but still returns a reading. In ICL a redundant demonstration does harm: show the LLM several examples from the same semantic template and its attention concentrates on that template, a form of in-context overfitting. The information-gain objective of Section~\ref{sec:objective} does not guard against this by itself. It is a pure exploration criterion, so unconstrained greedy maximization puts each annotation wherever the epistemic variance happens to be largest at the moment. Contrastively trained LLM embeddings have a ``hubness'' structure: queries that share lexical semantics (all the questions about ``sales'', say) collapse into dense regions, while rare SQL constructions such as correlated subqueries sit out in sparse tails. Where the rule goes on such embeddings depends on the pool. Posterior variance alone favors the sparse tails, but when one domain dominates the pool, most of the far-from-labeled, low-noise candidates lie inside it, and on our enterprise pool, whose seed domain holds 350 of the 477 examples, the unconstrained rule spends most of its budget there (Section~\ref{sec:ablation}). Neither concentration is what a bank needs. What it needs is one good demonstration per semantic stratum, and that is what the partition matroid enforces: at most one example per stratum, so the final prompt covers the skeleton of the SQL syntax space, with the objective choosing, inside each stratum, the query whose annotation is most informative and least noisy.

There is a tension here. The GP wants ambient coordinates, while the clusters are supposed to reflect the intrinsic geometry of the manifold. We resolve it by giving the two components different jobs: the matroid is responsible for global topology and the GP for local smoothness.

For the clusters $C_1, \dots, C_K$ we use the corpus's domain labels whenever they exist, which is the case in every experiment reported here. When they do not, we fall back on spectral clustering of a $k$-NN graph, which approximates intrinsic manifold distance rather than ambient Euclidean distance. Ambient Euclidean distance in high dimension suffers from the usual concentration problems and misses the way queries evolve structurally along the manifold; a Euclidean clustering can short-circuit across a sparse gap and lump together queries that are structurally different. Clustering on the manifold keeps the partition faithful to the data's topological skeleton.

The surrogate kernel $\mathcal{K}$, on the other hand, is an isotropic stationary kernel evaluated on a fixed low-dimensional projection of the ambient embedding (30 dimensions in every experiment; Section~\ref{sec:experiments}), not on a learned manifold metric. A manifold-specific or anisotropic kernel, automatic relevance determination for example, is attractive in principle but asks us to estimate thousands of length scales, and with $n < 500$ examples those cannot be estimated reliably. We rely instead on the fact that the embeddings are pre-trained. After L2-normalization the embeddings $\phi(q)$ lie on the unit sphere, where Euclidean distance is a monotone function of cosine similarity, and this is the metric that retrieval uses. The kernel itself is evaluated on the projected coordinates, where that exact relation no longer holds; the projection is built from the nearest-neighbor graph of the ambient embeddings, so it is local cosine neighborhoods, rather than distances, that it approximately preserves (see the scope discussion in Section~\ref{sec:icl_bridge}). On these coordinates a plain isotropic Mat\'ern kernel gives a tractable form of structural regularization: it depends only on local proximity and keeps the covariance matrix positive definite.

So the GP does local variance estimation and the matroid does global coverage, and between them the covariance matrix stays well-conditioned while the selection spans the SQL syntax space.

\begin{remark}[Projection and the guarantees]
\label{rem:projection}
The results of Section~\ref{sec:ideal_theory}, as well as Theorem~\ref{thm:np_hard_hetero_mi}, are stated for a finite pool equipped with a positive-definite kernel matrix $\Sigma_{\mathcal{C}\mathcal{C}}$ and a fixed partition. If $\Sigma_{\mathcal{C}\mathcal{C}}$ is computed from projected rather than ambient coordinates, a different matrix results, but monotonicity, submodularity, the $\tfrac{1}{2}$ guarantee, and the hardness reduction continue to hold, since each of them holds for an arbitrary positive-definite $\Sigma_{\mathcal{C}\mathcal{C}}$. The coordinates do matter in Section~\ref{sec:misspecified_theory}, where stationarity and the spectral constants are defined relative to the space in which the kernel is evaluated, and for this reason the diagnostics in Table~\ref{tab:thm3_diag} are computed in the 30-dimensional projection. They also matter in Section~\ref{sec:icl_bridge}, where retrieval and the kernel use different metrics. As for the partition, it is held fixed within each invocation of Algorithm~\ref{alg:greedy}, which is the unit to which Theorem~\ref{thm:matroid_ideal} applies, and in the experiments it is given by the domain labels, which do not change between rounds. The ablation in Section~\ref{sec:ablation} indicates that whether the projection is UMAP or PCA has no measurable effect on the selected sets once the partition constraint is in place.
\end{remark}

\subsection{Greedy Marginal Information Gain}

The candidate pool $\mathcal{C}$ is a finite set of queries, each represented by its embedding on the manifold $\mathcal{M}$. What the acquisition step needs is the marginal value of adding one more query to the current selection, and this has a closed form in the predictive variance of the surrogate.

\begin{lemma}[Marginal Information Gain via Posterior Variance]
\label{lemma:marginal_info_gain}
Let $F(\mathcal{S}) = I(Y_{\mathcal{S}}; f_{\mathcal{C}})$ be the information-gain set function computed under the assumed base covariance $\Sigma$ with heteroscedastic noise $\sigma^2(x) > 0$. Then $F(\mathcal{S}) = \frac{1}{2}\log\det(I + \Lambda_{\mathcal{S}}^{-1}\Sigma_{\mathcal{S}\mathcal{S}})$, and for any subset $\mathcal{S} \subset \mathcal{C}$ and candidate $x \in \mathcal{C} \setminus \mathcal{S}$ the marginal information gain of adding $x$ to $\mathcal{S}$ is exactly:
\begin{equation}
    F(\mathcal{S} \cup \{x\}) - F(\mathcal{S}) = \frac{1}{2}\log(\delta_x)
\end{equation}
where the variance ratio $\delta_x$ is
\begin{equation}
    \delta_x = \frac{ \Sigma_{xx} + \sigma^2(x) - \Sigma_{x\mathcal{S}} ( \Sigma_{\mathcal{S}\mathcal{S}} + \Lambda_{\mathcal{S}} )^{-1} \Sigma_{\mathcal{S}x} }{ \sigma^2(x) } = \frac{\mathrm{Var}(Y_x \mid Y_{\mathcal{S}})}{\sigma^2(x)} \;\ge\; 1, \label{eq:variance_ratio}
\end{equation}
and $\Lambda_{\mathcal{S}}$ is the diagonal matrix with entries $\sigma^2(x_i)$, $x_i \in \mathcal{S}$.
\end{lemma}

The numerator of $\delta_x$ is the predictive variance of the noisy response at $x$ given the annotations collected so far; the denominator is the irreducible noise at $x$. Equivalently, $\delta_x = 1 + \mathrm{Var}(f(x) \mid Y_{\mathcal{S}}) / \sigma^2(x)$, the signal-to-noise ratio of the epistemic uncertainty at $x$. The marginal gain is therefore large when the surrogate is uncertain about $f$ at $x$ and an annotation there would be reliable. It is small in two situations: when $x$ is already well predicted from $\mathcal{S}$ through the covariance structure, and when $\sigma^2(x)$ is large, so that even a fresh annotation would carry little information about $f(x)$. The second effect is the heteroscedastic penalty: as $\sigma^2(x)$ grows, $\delta_x \to 1$ and the marginal gain $\to 0$, so the algorithm steers away from queries with high irreducible ambiguity. Because $\delta_x \ge 1$, every marginal gain is non-negative, which is the monotonicity used in Section~\ref{sec:ideal_theory}.

Algorithm \ref{alg:greedy} outlines the greedy procedure. The score $\delta_x$ depends only on $\Sigma$ and $\Lambda$ restricted to $\mathcal{S} \cup \{x\}$; no inverse over the unselected pool is required, so with a Cholesky factor of $\Sigma_{\mathcal{S}\mathcal{S}} + \Lambda_{\mathcal{S}}$ maintained incrementally, each step costs $\mathcal{O}(|\mathcal{C}_{valid}| B^2)$. A set of still-available clusters enforces the partition constraint, so that each stratum contributes at most one query and the output lies in $\mathcal{I}$. The requirement $B \le K$ matters: when the current pool has fewer than $B$ non-empty strata, no feasible set of size $B$ exists. Our implementation then fills the remaining slots of the batch by unconstrained greedy selection on $\delta_x$. Such a batch lies outside $\mathcal{I}$, and the guarantees of Sections~\ref{sec:ideal_theory} and~\ref{sec:misspecified_theory} do not cover it; Section~\ref{sec:experiments} reports the rounds in which this happened.

\begin{algorithm}[H]
\caption{SHARP: Stratified Heteroscedastic Active Retrieval on a Partitioned manifold}
\label{alg:greedy}
\begin{algorithmic}[1]
\Require Assumed base covariance matrix $\Sigma_{\mathcal{C}\mathcal{C}}$, Heteroscedastic noise function $\sigma^2(x)$, Candidate pool $\mathcal{C}$ partitioned into $K$ non-empty semantic clusters $\{C_1, \dots, C_K\}$, batch budget $B \le K$.
\Ensure Selected query set $\mathcal{S} \subset \mathcal{C}$

\State $\mathcal{S} \gets \emptyset$
\State $\mathcal{J}_{valid} \gets \{1, \dots, K\}$ \Comment{Keep track of available cluster indices}

\For{$j = 1$ \textbf{to} $B$}
    \State $\mathcal{C}_{valid} \gets \bigcup_{i \in \mathcal{J}_{valid}} C_i$ \Comment{Pool of valid candidates}
    
    \For{$x \in \mathcal{C}_{valid}$}
        \State $\delta_x \gets \frac{ \Sigma_{xx} + \sigma^2(x) - \Sigma_{x\mathcal{S}} ( \Sigma_{\mathcal{S}\mathcal{S}} + \Lambda_{\mathcal{S}} )^{-1} \Sigma_{\mathcal{S}x} }{ \sigma^2(x) }$ \Comment{$\mathrm{Var}(Y_x \mid Y_{\mathcal{S}}) / \sigma^2(x)$; marginal gain is $\frac{1}{2}\log \delta_x$}
    \EndFor
    
    \State $x^* \gets \arg\max_{x \in \mathcal{C}_{valid}} \delta_x$
    \State $\mathcal{S} \gets \mathcal{S} \cup \{x^*\}$
    \State Let $k^*$ be the index of the cluster containing $x^*$
    \State $\mathcal{J}_{valid} \gets \mathcal{J}_{valid} \setminus \{k^*\}$ \Comment{Enforce Matroid Constraint}
\EndFor

\State \textbf{return} $\mathcal{S}$
\end{algorithmic}
\end{algorithm}

\noindent \textit{Note on Algorithm \ref{alg:greedy}:} $\Lambda_{\mathcal{S}}$ represents the diagonal matrix of heteroscedastic noise variances $\sigma^2(x)$ for the currently selected subset $\mathcal{S}$.
\section{Theoretical Results}
\label{sec:theory}

We now turn to what can be proved about Algorithm~\ref{alg:greedy}. Theorem~\ref{thm:np_hard_hetero_mi} showed that the exact problem is NP-hard and Lemma~\ref{lemma:marginal_info_gain} gave the greedy score in closed form. Here we show that greedy attains a constant-factor approximation when the kernel is right, work out how much of that survives when the kernel is wrong, and then connect the design objective to the downstream in-context learning task.
\subsection{Theoretical Guarantees: The Ideal Covariance Setting}
\label{sec:ideal_theory}

Suppose first that the kernel is correct, $\mathcal{K} = \mathcal{K}^*$. The greedy guarantee under a matroid needs the objective to be monotone and submodular, and the information gain $F(\mathcal{S}) = I(Y_{\mathcal{S}}; f_{\mathcal{C}})$ is both, query-dependent noise included.

\begin{lemma}[Monotonicity and Submodularity under Heteroscedasticity]
\label{lemma:submodularity}
Let $\mathcal{C}$ be a finite candidate pool of embedded queries with positive-definite kernel matrix $\Sigma_{\mathcal{C}\mathcal{C}}$, and let $Y_x = f(x) + \epsilon(x)$ be the observations modeled by a Gaussian Process with independent heteroscedastic noise $\sigma^2(x) > 0$. The information-gain set function $F(\mathcal{S}) = I(Y_{\mathcal{S}}; f_{\mathcal{C}})$ is normalized ($F(\emptyset) = 0$), monotone non-decreasing, and submodular on $2^{\mathcal{C}}$.
\end{lemma}

The monotonicity part of the lemma is where the present formulation departs from the mutual-information criterion $I(Y_{\mathcal{S}}; Y_{\mathcal{C} \setminus \mathcal{S}})$ used by \citet{krause2008near}. That criterion is submodular, but on a finite pool it fails to be monotone, since adding a query to $\mathcal{S}$ removes it from the complement. The usual remedy is $\epsilon$-approximate monotonicity, which requires the pool to be a fine discretization of a continuous domain so that every selected point has an unselected near-duplicate, and a finite pool of real user queries does not have this property at the scale of the kernel length scale. Conditioning on the latent field $f_{\mathcal{C}}$ rather than on the noisy complement avoids the difficulty altogether, because by Lemma~\ref{lemma:marginal_info_gain} every marginal gain equals $\frac{1}{2}\log \delta_x \ge 0$.

Given monotonicity and submodularity, the performance of greedy selection under the diversity constraint follows from the classical matroid result of \citet{fisher1978analysis}.

\begin{theorem}[Approximation Guarantee under Matroid Constraints]
\label{thm:matroid_ideal}
Let $\mathcal{S}_{greedy}$ be the set of queries selected by Algorithm~\ref{alg:greedy}, and let $\mathcal{S}^* = \arg\max_{\mathcal{S} \in \mathcal{I}} F(\mathcal{S})$ be an optimal selection subject to the partition matroid constraint (at most one query per semantic cluster and at most $B$ in total). Under the assumptions of Lemma~\ref{lemma:submodularity}, the greedy acquisition satisfies:
\begin{equation}
    F(\mathcal{S}_{greedy}) \geq \frac{1}{2}\, F(\mathcal{S}^*).
\end{equation}
\end{theorem}

In words: with the right kernel, the greedy selection attains at least half the information gain of the best selection satisfying the same partition constraint. The comparison is with the best stratified set, not with the best unconstrained one, so the theorem says nothing about what the stratification itself costs. It holds on the finite pool as it stands, with no discretization argument and no approximate-monotonicity penalty.

\begin{remark}[Comparison to Cardinality-Constrained Bounds]
The approximation factor of $\frac{1}{2}$ in Theorem \ref{thm:matroid_ideal} is weaker than the classic $(1 - 1/e)$ bound for greedy maximization of a monotone submodular function under a cardinality constraint \citep{nemhauser1978analysis}, which is the setting of \citet{krause2008near}. A cardinality constraint is a uniform matroid; enforcing spatial diversity across semantic clusters requires a partition matroid, and under a general matroid the plain greedy algorithm is guaranteed only a $\frac{1}{2}$ approximation \citep{fisher1978analysis}. The $(1 - 1/e)$ factor is recoverable under any matroid by the continuous greedy algorithm with pipage rounding \citep{calinescu2011maximizing}, at the cost of evaluating the multilinear extension of $F$ by sampling. We use the plain greedy because the $\frac{1}{2}$ bound then applies to the algorithm actually deployed. The price of the diversity guarantee, in the worst case, is the drop from about $0.632$ to $0.5$.
\end{remark}

\paragraph{Relation to \citet{krause2008near}.} The closest prior work is the sensor-placement analysis of \citet{krause2008near}, and it is worth being precise about what is borrowed from it. The GP surrogate, greedy selection, and the use of submodularity to get a constant-factor guarantee are common to both. The differences are these. (i)~The objective is the information gain $I(Y_{\mathcal{S}}; f_{\mathcal{C}})$ rather than the mutual information $I(Y_{\mathcal{S}}; Y_{\mathcal{C} \setminus \mathcal{S}})$; the former is exactly monotone on any finite pool, the latter only approximately monotone on a sufficiently fine discretization of a continuous domain. With homoscedastic noise $\sigma^2$ the former differs from the entropy criterion $H(Y_{\mathcal{S}})$, which \citet{krause2008near} also study, by $\frac{|\mathcal{S}|}{2}\log(2\pi e \sigma^2)$, a term that depends on $\mathcal{S}$ only through its size and is therefore constant across sets of a fixed size. (ii)~The constraint is a partition matroid rather than a cardinality bound, which changes the greedy factor from $1 - 1/e$ to $\frac{1}{2}$. (iii)~The noise is heteroscedastic and enters both the objective and the acquisition score. (iv)~The ground set is a finite pool of real user queries rather than a discretization of a domain, which is why none of the discretization-based arguments of \citet{krause2008near} are used. (v)~The misspecification analysis of Section~\ref{sec:misspecified_theory} has no counterpart in their work, which assumes the kernel is known. Items (i) through (iv) adapt a known framework to a new setting. Item (v) is new, but Section~\ref{sec:misspecified_theory} will show that its constants are too large on our pools to say anything quantitative. We would not point to any single theorem as the contribution. The contribution is the formulation as a whole, the bridge to the downstream task in Section~\ref{sec:icl_bridge}, and the evaluation on a production system in Section~\ref{sec:experiments}.

\subsection{Theoretical Guarantees: Robustness to Covariance Misspecification}
\label{sec:misspecified_theory}

Section \ref{sec:ideal_theory} assumed the true covariance $\mathcal{K}^*$, which is not available in practice; Algorithm~\ref{alg:greedy} runs on a surrogate $\Phi$, a Mat\'ern kernel in our case. Does Theorem~\ref{thm:matroid_ideal} survive when $\Phi \neq \mathcal{K}^*$? In form, yes; in magnitude, not on our pools. The greedy selection loses an additive discrepancy at each step, and below we make that discrepancy explicit, identify what it depends on, and then compute it for our two pools.

Throughout this section both kernels are stationary on $\mathbb{R}^d$ with spectral densities $\hat{\Phi}$ and $\hat{\mathcal{K}}^*$. We write $F_\Phi$ and $F_{\mathcal{K}^*}$ for the objective of Equation~\eqref{eq:objective_closed} computed under each kernel, $\Delta F_\Phi$ and $\Delta F_{\mathcal{K}^*}$ for the corresponding marginal gains of Lemma~\ref{lemma:marginal_info_gain}, and
\begin{equation}
    P_{\mathcal{K}}(x \mid \mathcal{S}) = \mathrm{Var}_{\mathcal{K}}\big(f(x) \mid Y_{\mathcal{S}}\big) = \mathrm{Var}_{\mathcal{K}}(Y_x \mid Y_{\mathcal{S}}) - \sigma^2(x) \label{eq:power}
\end{equation}
for the posterior variance of the latent value at $x$ under kernel $\mathcal{K}$, so that $\Delta F_{\mathcal{K}}(x \mid \mathcal{S}) = \frac{1}{2}\log\big(1 + P_{\mathcal{K}}(x \mid \mathcal{S})/\sigma^2(x)\big)$. Note that $P_{\mathcal{K}}$ depends only on the kernel and on the design $(\mathcal{S}, x)$, and not on the response surface $f$ or its smoothness.

\begin{assumption}[Spectral Misspecification]
\label{assum:spectral}
Let $\beta(\omega) = \hat{\Phi}(\omega) / \hat{\mathcal{K}}^*(\omega)$ be the spectral ratio of the assumed to the true kernel.
\begin{enumerate}
    \item[(a)] There exist constants $c_1, c_2 > 0$ with $c_1 \le 1 \le c_2$ and an exponent $s \ge 0$ such that
    \begin{equation}
        c_1 \leq \beta(\omega) \leq c_2(1 + \|\omega\|^2)^s \qquad \text{for all } \omega \in \mathbb{R}^d.
    \end{equation}
    \item[(b)] The true kernel has a finite spectral moment of order $s$:
    \begin{equation}
        m_s(\mathcal{K}^*) = \int_{\mathbb{R}^d} \big[(1 + \|\omega\|^2)^s - 1\big]\, \hat{\mathcal{K}}^*(\omega)\, d\omega < \infty.
    \end{equation}
\end{enumerate}
\end{assumption}

Condition (a) describes the usual Kriging misspecification regime, in which the assumed kernel $\Phi$ neither dominates nor is dominated by $\mathcal{K}^*$ uniformly, and the polynomial factor allows $\Phi$ to be rougher than $\mathcal{K}^*$ by up to $s$ derivatives. Condition (b) is automatic when $s = 0$. For a Mat\'ern true kernel with smoothness $\nu^*$, whose spectral density decays as $(1 + \|\omega\|^2)^{-(\nu^* + d/2)}$, condition (b) holds if and only if $s < \nu^*$. The pairing used in Section~\ref{sec:experiments}, a Mat\'ern-$1/2$ surrogate against a Mat\'ern-$\nu^*$ truth with $\nu^* \ge 1/2$ and a common length scale, satisfies (a) with $s = \nu^* - 1/2$ and constants determined by the ratio of normalizing factors, and therefore satisfies (b). This is the direction of misspecification in which the surrogate is at least as rough as the truth.

Condition (b) makes the following object well defined. Let $\mathcal{K}^*_s$ denote the stationary kernel with spectral density $(1 + \|\omega\|^2)^s\, \hat{\mathcal{K}}^*(\omega)$, the true kernel roughened by $s$ derivatives. Define the \emph{roughening residual} and its worst case over the pool by
\begin{equation}
    R_s(x \mid \mathcal{S}) = P_{\mathcal{K}^*_s}(x \mid \mathcal{S}) - P_{\mathcal{K}^*}(x \mid \mathcal{S}) \;\ge\; 0, \qquad
    \bar{R}_s = \max_{\mathcal{S} \in \mathcal{I},\; x \in \mathcal{C} \setminus \mathcal{S}} R_s(x \mid \mathcal{S}). \label{eq:roughening_residual}
\end{equation}
The residual is the additional posterior variance the true model would carry if its kernel were roughened by $s$ derivatives. When $s = 0$ it vanishes identically; on a finite pool it is a maximum over finitely many pairs, and for any hypothesized $(\mathcal{K}^*, s)$ it can be computed from the pool geometry alone. Let $\sigma^2_{\min} = \min_{x \in \mathcal{C}} \sigma^2(x) > 0$ denote the noise floor.

Since Algorithm~\ref{alg:greedy} selects queries by the marginal gain computed under the \emph{assumed} kernel, we need to relate assumed and true marginal gains in both directions. An upper bound controls how optimistic the surrogate can be about the query it picks, and a lower bound controls how pessimistic it can be about the query it passes over.

\begin{lemma}[Two-Sided Information Gain Mismatch]
\label{lemma:variance_mismatch}
Under Assumption~\ref{assum:spectral}, for every $\mathcal{S} \subseteq \mathcal{C}$ and $x \in \mathcal{C} \setminus \mathcal{S}$:
\begin{enumerate}
    \item[(i)] $c_1\, P_{\mathcal{K}^*}(x \mid \mathcal{S}) \;\le\; P_\Phi(x \mid \mathcal{S}) \;\le\; c_2 \big[ P_{\mathcal{K}^*}(x \mid \mathcal{S}) + R_s(x \mid \mathcal{S}) \big]$;
    \item[(ii)] $\Delta F_{\mathcal{K}^*}(x \mid \mathcal{S}) \;\le\; \Delta F_\Phi(x \mid \mathcal{S}) + \tfrac{1}{2}\log(1/c_1)$;
    \item[(iii)] if $c_2 \bar{R}_s < \sigma^2_{\min}$, then
    $\Delta F_{\mathcal{K}^*}(x \mid \mathcal{S}) \;\ge\; \Delta F_\Phi(x \mid \mathcal{S}) - \tfrac{1}{2}\log c_2 - \varepsilon_s$, where
    \begin{equation}
        \varepsilon_s = -\tfrac{1}{2}\log\!\Big(1 - \frac{c_2 \bar{R}_s}{\sigma^2_{\min}}\Big) \;\ge\; 0, \qquad \varepsilon_0 = 0.
    \end{equation}
\end{enumerate}
Consequently, with $\eta = \frac{1}{2}\log(c_2/c_1) + \varepsilon_s$, the assumed and true marginal gains differ by at most $\eta$ in either direction.
\end{lemma}

The hypothesis of part (iii) is a genuine restriction, though a checkable one: it asks that the roughening residual, scaled by $c_2$, sit below the noise floor. When it fails the surrogate's marginal gains can be arbitrarily optimistic and no additive bound of this kind exists. That is the regime of a surrogate far too rough for the data.

By Lemma~\ref{lemma:variance_mismatch}, Algorithm~\ref{alg:greedy}, although it maximizes $\Delta F_\Phi$, behaves as an $\eta$-approximate greedy algorithm for $\Delta F_{\mathcal{K}^*}$, in the sense that at each step it picks a query whose true marginal gain is within $\eta$ of the best available true marginal gain. Approximate greedy selection under a matroid constraint inherits the guarantee of Theorem~\ref{thm:matroid_ideal} up to the accumulated error.

\begin{theorem}[Robust Approximation Guarantee under Matroid Constraints]
\label{thm:robust_matroid}
Let $\mathcal{S}_{greedy}$ be the set of $B$ queries selected by Algorithm~\ref{alg:greedy} using the surrogate kernel $\Phi$, and let $OPT_{\mathcal{K}^*} = \max_{\mathcal{S} \in \mathcal{I}} F_{\mathcal{K}^*}(\mathcal{S})$ be the maximum information gain achievable under the true kernel subject to the partition matroid constraint. Under Assumption~\ref{assum:spectral} and the hypothesis $c_2 \bar{R}_s < \sigma^2_{\min}$ of Lemma~\ref{lemma:variance_mismatch}(iii),
\begin{equation}
    F_{\mathcal{K}^*}(\mathcal{S}_{greedy}) \;\geq\; \frac{1}{2}\Big( OPT_{\mathcal{K}^*} - B\eta \Big), \qquad \eta = \tfrac{1}{2}\log(c_2/c_1) + \varepsilon_s. \label{eq:robust_bound}
\end{equation}
\end{theorem}

A word about the constant in Theorem~\ref{thm:robust_matroid}. For a fixed finite pool and a fixed budget $B$, $\eta$ is a fixed positive number and does not vanish. The bound is informative when $B\eta \ll OPT_{\mathcal{K}^*}$, which requires the spectral ratio $c_2/c_1$ to be near one and the roughening residual to be small relative to the noise floor. Both quantities can be estimated for any hypothesized true kernel, $c_2/c_1$ from the two spectral densities and $\bar{R}_s$ from the pool geometry. When $s = 0$ the residual vanishes identically and the theorem reduces to $F_{\mathcal{K}^*}(\mathcal{S}_{greedy}) \ge \frac{1}{2}\big(OPT_{\mathcal{K}^*} - \frac{B}{2}\log(c_2/c_1)\big)$, a pure spectral penalty. In the noiseless interpolation limit, $P_{\mathcal{K}^*}$ and $P_{\mathcal{K}^*_s}$ are the squared power functions of $\mathcal{K}^*$ and $\mathcal{K}^*_s$, which decay with the fill distance $h_{\mathcal{S}, \mathcal{M}} = \sup_{x \in \mathcal{M}} \min_{x_i \in \mathcal{S}} d_{\mathcal{M}}(x, x_i)$ of the design at rates $h^{2\nu^*}$ and $h^{2(\nu^* - s)}$ respectively \citep[Ch.~11]{wendland2005scattered}, with the intrinsic dimension replacing the ambient one on a compact manifold \citep{fuselier2012scattered}; so that $\bar{R}_s \to 0$ as the annotated set fills the manifold, provided $s < \nu^*$. At a finite annotation budget this asymptotic regime is not reached, however, and $\bar{R}_s$ has to be treated as a constant.

\paragraph{Numerical diagnostics.} Every quantity in Theorem~\ref{thm:robust_matroid} can be computed once a true kernel is hypothesized, and Appendix~\ref{app:thm3_diag} does so for the configuration of Section~\ref{sec:experiments} under four hypothesized truths (Table~\ref{tab:thm3_diag}). For every one of them the right side of Equation~\eqref{eq:robust_bound} is negative. When $s > 0$, the hypothesis of Lemma~\ref{lemma:variance_mismatch}(iii) fails by between six and twelve orders of magnitude; when $s = 0$, the spectral penalty alone gives $B\eta$ between $33$ and $54$ nats against a per-batch objective below $9$ nats. This is not a numerical artifact: the ratio $c_2/c_1$ grows exponentially in the dimension for any two distinct Mat\'ern kernels, and $\bar{R}_s$ is of order $(d/\ell^2)^s$ unless $\ell \gg \sqrt{d}$.

We read Theorem~\ref{thm:robust_matroid} as a structural result, then, and not as a certificate for the selections reported later. It says which two quantities govern the degradation, the spectral ratio and the roughening residual relative to the noise floor, and it says that the degradation is additive and per-step rather than multiplicative. What it does not do, at our dimensions and budgets, is bound the true information gain of the selected sets by any useful margin. Whatever robustness the algorithm shows in the experiments rests on the experiments themselves and on the ablation in Section~\ref{sec:ablation}, which traces the gains to the partition constraint, a component built from domain labels that never touches the kernel.

\subsection{From Information Gain to In-Context Learning Utility}
\label{sec:icl_bridge}

The objective $F$ measures information about the latent correctness surface $f$. The task is something else: to build a bank whose examples, retrieved as demonstrations for future queries, make the LLM more likely to write correct SQL. This subsection connects the two. We set out a model of the retrieval step, derive the lower bound on downstream accuracy it implies, show in what sense the greedy rule controls the quantities in that bound, and calibrate the model's one substantive assumption on both pools. We also try to be clear about which parts of the argument carry weight.

\paragraph{Retrieval model.} A test query with embedding $x \sim \nu$ retrieves its nearest annotated example $x_{\mathcal{S}}(x) = \arg\min_{x' \in \mathcal{S}} d(x, x')$ in the ambient embedding metric $d$, as in the evaluation of Section~\ref{sec:experiments}. Let $a(x, x') \in [0, 1]$ be the expected SemanticScore of the SQL the LLM writes for $x$ when shown $x'$ and its gold SQL as the demonstration, and let $A(\mathcal{S}) = \mathbb{E}_{x \sim \nu}\, a\big(x, x_{\mathcal{S}}(x)\big)$ be the downstream accuracy of the bank $\mathcal{S}$. Write $\bar{a}(x) = \sup_{x'} a(x, x')$ for the best achievable accuracy at $x$ and $\bar{A} = \mathbb{E}_\nu \bar{a}(x)$.

\begin{assumption}[Noise-Aware Locality]
\label{assum:locality}
There exist $L_a, L_\sigma \ge 0$ and a locality radius $r_0 > 0$ such that for all $x$ and all $x'$ with $d(x, x') \le r_0$,
\begin{equation}
    a(x, x') \;\ge\; \bar{a}(x) - L_a\, d(x, x') - L_\sigma\, \sigma^2(x'). \label{eq:locality}
\end{equation}
\end{assumption}

The distance term is the premise behind every retrieval-based few-shot system: a demonstration is worth less the further it sits from the query in embedding space. It is also why Dist-to-Threshold is a sensible baseline. The noise term says something more specific. A demonstration drawn from a query that the LLM itself resolves inconsistently, one with large self-consistency variance $\sigma^2(x')$, teaches a less stable question-to-SQL mapping and should be worth less than a stable one at the same distance. If the noise term were absent, Assumption~\ref{assum:locality} would make fill distance the only relevant quantity and would support farthest-point ($k$-center) selection rather than Algorithm~\ref{alg:greedy}. The heteroscedastic objective exists to trade this term against distance, and of the two terms it is the one that most needs empirical support. We come back to it in the calibration below.

\begin{proposition}[Downstream Accuracy Lower Bound]
\label{prop:icl_bound}
Let $\mathcal{C} = C_1 \cup \dots \cup C_K$ be the domain partition and $\nu = \sum_k \pi_k \nu_k$ with $\nu_k$ supported on domain $k$. Suppose $\mathcal{S}$ contains at least one example in every domain, and let $h_k(\mathcal{S}) = \sup_{x \in \mathrm{supp}\, \nu_k} d(x, \mathcal{S} \cap C_k)$ be the fill distance of $\mathcal{S}$ within domain $k$. Under Assumption~\ref{assum:locality}, whenever $h_k(\mathcal{S}) \le r_0$ for all $k$,
\begin{equation}
    A(\mathcal{S}) \;\ge\; \bar{A} \;-\; L_a \sum_{k=1}^K \pi_k\, h_k(\mathcal{S}) \;-\; L_\sigma \max_{x' \in \mathcal{S}} \sigma^2(x'). \label{eq:icl_bound}
\end{equation}
If some domain $k$ has $\mathcal{S} \cap C_k = \emptyset$, then the retrieved demonstration for every query in that domain lies outside it, and the bound for that domain degrades to $\pi_k$ times the cross-domain floor $\inf_{x \in C_k, x' \notin C_k} a(x, x')$.
\end{proposition}

\begin{proof}
For $x$ in domain $k$, the retrieved example satisfies $d(x, x_{\mathcal{S}}(x)) \le d(x, \mathcal{S} \cap C_k) \le h_k(\mathcal{S}) \le r_0$, so Equation~\eqref{eq:locality} applies with $d(x, x_{\mathcal{S}}(x)) \le h_k(\mathcal{S})$ and $\sigma^2(x_{\mathcal{S}}(x)) \le \max_{\mathcal{S}} \sigma^2$. Take expectations over $\nu_k$ and sum with weights $\pi_k$. The uncovered-domain statement is immediate from the definition of $x_{\mathcal{S}}(x)$.
\end{proof}

Proposition~\ref{prop:icl_bound} is elementary. Its use is that it isolates the three quantities a selection rule can influence (per-domain fill distance, demonstration noise, and domain coverage) and makes the role of the partition visible. When the partition is the domain labeling, as it is in all of our experiments, the constraint is what guarantees that every domain present in the pool is represented after the first batch when the pool has at most $B$ domains, and within $\lceil K/B \rceil$ batches in general. Without it, the uncovered-domain clause applies to whichever strata the unconstrained rule happens to ignore, and on the enterprise pool that turns out to be most of them (Section~\ref{sec:ablation}).

\paragraph{What the greedy score controls.} The two remaining terms are influenced by the score of Algorithm~\ref{alg:greedy}. Write $\delta_x = \mathrm{Var}(Y_x \mid Y_{\mathcal{S}}) / \sigma^2(x) = 1 + P_\Phi(x \mid \mathcal{S}) / \sigma^2(x)$ for the marginal-gain score of Lemma~\ref{lemma:marginal_info_gain}. The posterior variance in this expression is bounded in terms of the distance to the nearest selected point.

\begin{lemma}[One-Neighbor Variance Bound]
\label{lemma:one_neighbor}
Let $\Phi$ be a stationary kernel with $\Phi(x, x) = 1$ and radial profile $\varphi(r)$, and let $x_j = x_{\mathcal{S}}(x)$ be the nearest selected point to $x$ at distance $r = d(x, \mathcal{S})$. Then
\begin{equation}
    P_\Phi(x \mid \mathcal{S}) \;\le\; 1 - \frac{\varphi(r)^2}{1 + \sigma^2(x_j)}. \label{eq:one_neighbor}
\end{equation}
For the Mat\'ern-$1/2$ kernel with length scale $\ell$ the right side is $1 - e^{-2r/\ell} / (1 + \sigma^2(x_j))$.
\end{lemma}

\begin{proof}
Posterior variance under a Gaussian process is non-increasing in the conditioning set, so $P_\Phi(x \mid \mathcal{S}) \le P_\Phi(x \mid \{x_j\}) = 1 - \Phi(x, x_j)^2 / (\Phi(x_j, x_j) + \sigma^2(x_j))$, and $\Phi(x, x_j) = \varphi(r)$.
\end{proof}

By Lemma~\ref{lemma:one_neighbor}, the score $\delta_x$ is bounded by a function that increases with the distance from $x$ to the current selection and decreases with the noise $\sigma^2(x)$ at $x$, so the greedy step can be read as a farthest-point rule weighted by noise. With homoscedastic noise the weighting disappears, and in the noiseless limit the rule becomes the $P$-greedy algorithm of kernel approximation, which selects the point at which the power function is largest. For $P$-greedy on a compact domain, \citet{santin2017convergence} show that the fill distance of the selected points decays at the optimal rate for the kernel's smoothness class, and \citet{wenzel2021novel} show that the points become asymptotically uniformly distributed, with fill distance and separation distance of the same order. That is the sense in which maximizing information gain controls the fill-distance term in Equation~\eqref{eq:icl_bound}. The heteroscedastic weighting then gives up some fill distance in exchange for lower demonstration noise, which is the second term of the same bound. We have no comparable rate for the heteroscedastic case.

\paragraph{Calibration on the two pools.} Assumption~\ref{assum:locality} concerns the LLM and cannot be checked directly without a label for every pair of queries. The retrieval simulation of Section~\ref{sec:experiments} offers a natural proxy, $a(x, x') \approx \mathrm{SemanticScore}(s_{x'}, s_x)$, the score of the demonstration's gold SQL against the query's gold SQL, which is the quantity the simulation rewards. We computed this proxy for every ordered pair in each pool (227{,}052 enterprise pairs and 3{,}966{,}072 BIRD pairs), together with the ambient embedding distance and, for the enterprise pool, the self-consistency prior $\sigma^2(x')$ of the demonstration. Figure~\ref{fig:icl_locality} shows the binned means.

\begin{figure}[h]
\centering
\includegraphics[width=0.95\textwidth]{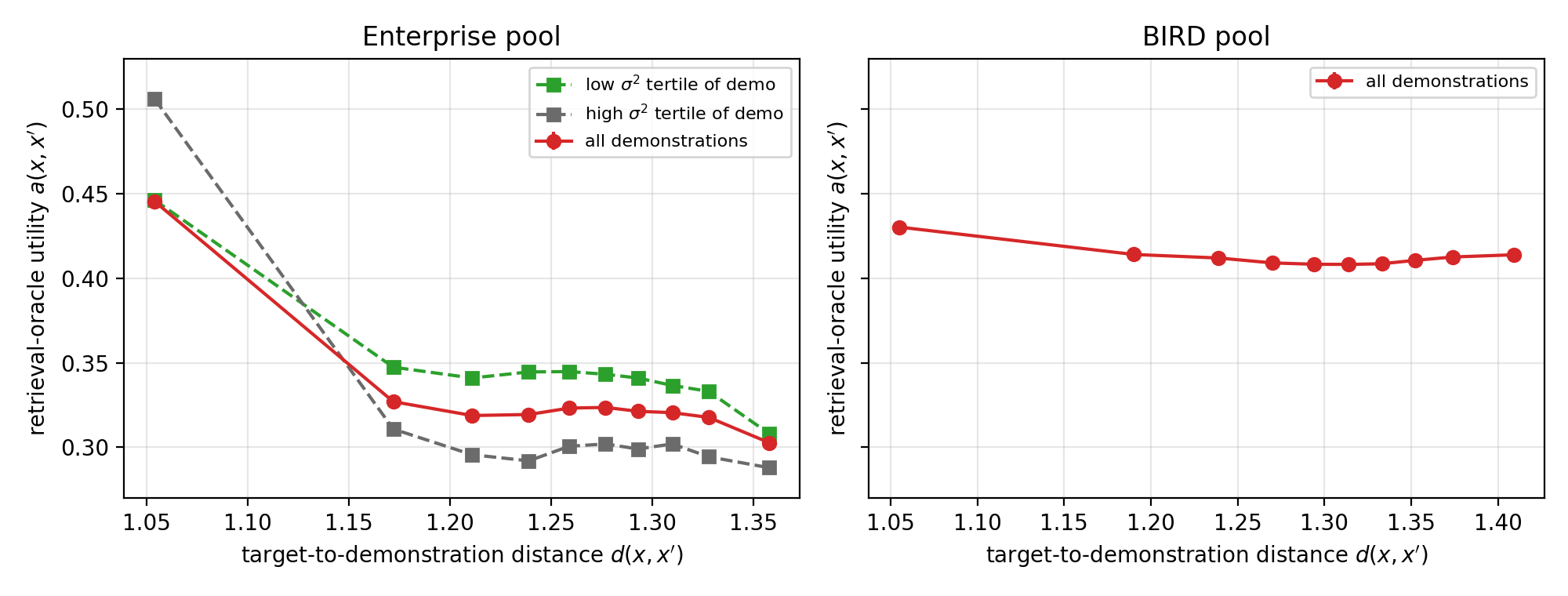}
\caption{Calibration of Assumption~\ref{assum:locality} with the retrieval-oracle proxy $a(x, x') = \mathrm{SemanticScore}(s_{x'}, s_x)$. Mean utility by decile of the target-to-demonstration distance ($\pm 1$ SE), all ordered pairs. Enterprise (left) also splits demonstrations by tertile of their self-consistency noise $\sigma^2(x')$; the high-noise curve's nearest-decile value is confounded with the dense seed domain. Both panels share the vertical scale. Utility is high only in the nearest decile and flat beyond it; the near-to-far gap is $0.12$ on the enterprise pool and $0.02$ on BIRD.}
\label{fig:icl_locality}
\end{figure}

Locality holds in this proxy, but it is short-ranged. On the enterprise pool the mean utility is $0.445$ in the nearest distance decile and between $0.30$ and $0.33$ in every other decile, and on BIRD it is $0.430$ in the nearest decile and about $0.41$ elsewhere. A linear fit over all pairs is accordingly poor ($R^2 = 0.06$ and $0.002$). Within the retrieval regime, defined as each query's ten nearest examples, the linear form of Equation~\eqref{eq:locality} fits with $L_a = 0.83 \pm 0.02$ ($R^2 = 0.22$) on the enterprise pool and $L_a = 0.24 \pm 0.01$ ($R^2 = 0.03$) on BIRD, in units of ambient distance on unit-norm embeddings. The locality radius $r_0$ corresponds to the boundary of the nearest decile. Proposition~\ref{prop:icl_bound} requires every per-domain fill distance to lie inside that radius, and with one to ten annotated examples per domain the banks in our experiments are unlikely to meet this requirement; we have not verified it. As with Theorem~\ref{thm:robust_matroid}, the proposition identifies the terms that govern downstream accuracy at our budgets without bounding it quantitatively. The gap between near and far utility, $0.12$ on the enterprise pool and $0.02$ on BIRD, is the amount that Equation~\eqref{eq:icl_bound} says a coverage-improving rule can recover. The retrieval-simulation gaps of Sections~\ref{sec:experiments} and~\ref{sec:bird} are of the same order, roughly $0.06$ to $0.12$ on the enterprise pool and within $0.02$ on BIRD. This agreement is partly by construction, since the proxy used here is the same gold-SQL SemanticScore that the retrieval simulation rewards; the end-to-end LLM evaluations are the independent check. The noise term receives weaker support from the proxy. Over all pairs the coefficient on $\sigma^2(x')$ is $-0.37 \pm 0.01$, which has the sign the assumption requires, and in every distance decile beyond the first, demonstrations in the lowest noise tertile score about $0.04$ higher than those in the highest. Within the ten-nearest regime, however, the coefficient is $+0.19 \pm 0.06$, of the opposite sign, and the nearest decile is confounded with the dense seed domain. We therefore treat $L_\sigma > 0$ as motivated by the noise model of Section~\ref{sec:objective} and consistent with the aggregate evidence, without claiming that the proxy establishes it; the proxy measures the gold SQL rather than the LLM's response to an ambiguous demonstration, and only the end-to-end evaluation speaks to the latter.

\paragraph{Predictions, checked after the fact.} Put together, Proposition~\ref{prop:icl_bound}, Lemma~\ref{lemma:one_neighbor}, and the calibration predict which ingredients of a selection rule should matter. First, covering every domain should recover most of the near-to-far utility gap, whatever rule picks the example inside each domain. Second, reducing fill distance within domains should add to this in proportion to $L_a$, so it should help on the enterprise pool ($L_a = 0.83$) and hardly at all on BIRD ($L_a = 0.24$), and more at small budgets, where per-domain fill distances are still large. Third, since the calibration gives $L_\sigma$ only weak and partly contradictory support, the noise weighting that separates Algorithm~\ref{alg:greedy} from a farthest-point rule should have little effect. The stratified controls of Section~\ref{sec:strat_controls} bear out all three: a uniform draw inside each domain matches SHARP on both pools in the oracle-label evaluations; a stratified farthest-point rule beats the uniform draw at small budgets on the enterprise pool but not on BIRD; and SHARP does no better than the farthest-point rule. We wrote these predictions down after the experiments had been run, so they are a consistency check on the retrieval model, not a test of it.

\paragraph{Scope of the argument.} What this subsection gives is a lower bound on downstream accuracy in terms of per-domain fill distance, demonstration noise, and coverage; an argument that the greedy score moves the first two of these in the right direction, with a rate available in the homoscedastic limit; and a calibration indicating that the distance term of the assumption holds in the retrieval regime, with constants that anticipate where coverage and within-domain spreading should and should not help. It does not provide a guarantee on $A(\mathcal{S}_{greedy})$ with known constants, a rate under heteroscedastic noise, or a model of how the LLM combines several retrieved demonstrations, and its noise term rests more on the noise model than on the proxy. The two halves of the argument are also stated in different metrics in our implementation: retrieval, and hence Proposition~\ref{prop:icl_bound}, uses ambient embedding distance, whereas the kernel, and hence Lemma~\ref{lemma:one_neighbor}, is evaluated in the 30-dimensional projection, which does not preserve distances. The fill-distance control of Lemma~\ref{lemma:one_neighbor} therefore transfers to Proposition~\ref{prop:icl_bound} only to the extent that the projection preserves nearest-neighbor structure. The empirical sections should be read with these limits in mind.

\section{Empirical Validation}
\label{sec:experiments}

We test SHARP in two settings. Case~1 is the production enterprise system that motivated the work. Case~2 is a public benchmark, included so that the results can be reproduced and so that proprietary data can be ruled out as the explanation.

\subsection{Case 1: Enterprise Supply-Chain Application}
\label{sec:enterprise}

The production system is a text-to-SQL assistant for supply-chain analytics at NVIDIA whose queries fall into seven domains, each covering a different aspect of the supply chain. Natural language questions are
embedded using \texttt{nvidia/\allowbreak llama-3.2\allowbreak
-nv\allowbreak -embedqa\allowbreak -1b\allowbreak -v2} ($d = 2048$) via a
NVIDIA NIM endpoint and stored in a Milvus vector database.

\subsubsection{Mapping Paper Notation to the Experiment}

Before the results, here is how the notation of Sections~\ref{sec:formulation}--\ref{sec:theory} maps onto the objects in the experiment.

\paragraph{Input space $x = \phi(q)$.} Each natural language query $q$ (e.g., ``Which suppliers have the highest open PO value?'') is mapped to a point $x = \phi(q) \in \mathbb{R}^{2048}$ by the NIM embedding model. This $x$ is the embedded query that all GP operations act on. Because $d = 2048$ is too large for a numerically well-conditioned GP covariance matrix, we project onto a 30-dimensional UMAP embedding (UMAP-30; \citealt{mcinnes2018umap}) before GP fitting, yielding $\tilde{x} \in \mathbb{R}^{30}$. UMAP constructs a weighted $k$-nearest-neighbor graph over the ambient embeddings and finds a low-dimensional layout that minimizes cross-entropy between the fuzzy topological representations in ambient and reduced space, so it targets geodesic structure rather than global variance. PCA, by contrast, finds the linear subspace of maximal variance and pays no attention to nonlinear manifold geometry; UMAP's coordinates are closer to a low-dimensional chart of $\mathcal{M}$. The Levina--Bickel estimate $\hat{d}_I \approx 18.4$ (Section~\ref{sec:experiments_dim}) suggests that 30 dimensions is enough.

\paragraph{Latent function $f(x)$ and observations $y$.} The unknown function $f : \mathcal{M} \to [0,1]$ from Section~\ref{sec:formulation} represents the expected SQL correctness of the annotated bank at embedding location $x$. To measure retrieval quality separately from the LLM's generative ability, we proxy $f(x_i)$ by the maximum SQL Semantic Match Score between example $i$'s SQL and any query in the held-out test set $\mathcal{T}$:
\begin{equation}
    \tilde{y}_i = \max_{t \in \mathcal{T}} \text{SemanticScore}(\text{sql}_i, \text{sql}_t)
\end{equation}
The acquisition score $\delta_x$ of Algorithm~\ref{alg:greedy} is a function of the input locations $X$, the kernel, and the heteroscedastic noise $\sigma^2(x)$ only (Remark~\ref{rem:label_dependence}). In the experiments of this section the kernel length scale is set by a label-free median heuristic and $\sigma^2(x)$ by the label-free self-consistency prior described below, so the selection of queries does not depend on the label values $y$ at all; labels are used for evaluation and, in Section~\ref{sec:realistic}, for the GP posterior mean, but not for selection. Defining the oracle $\tilde{y}_i$ from the test set therefore leaks nothing into selection; $\tilde{y}_i$ serves only as an offline evaluation metric in the retrieval simulation, an upper bound on what the selected bank could deliver. (We use $\max$ rather than the mean so that overrepresented domains do not inflate the metric.)

In the end-to-end LLM evaluation, $y$ is instead the actual SemanticScore of the SQL generated by the LLM given the retrieved few-shot bank, which is the real downstream quantity of interest.

\paragraph{Heteroscedastic noise $\sigma^2(x)$.} In Section~\ref{sec:formulation} we modeled each observation as $y_i = f(x_i) + \epsilon(x_i)$ with $\epsilon(x_i) \sim \mathcal{N}(0, \sigma^2(x_i))$, where $\sigma^2(x_i)$ captures the irreducible aleatoric uncertainty at $x_i$.  Since $\sigma^2(x)$ is unobservable before labeling, we estimate it in a pre-experiment pass using LLM self-consistency: the pool is spectral-clustered into $C=30$ semantic groups, five points are sampled uniformly from each cluster, and the LLM is prompted $M=5$ times at temperature $T=0.7$ with zero few-shot examples (schema only) for each sampled point.  Rather than scoring generated SQL against ground truth (which gives near-zero variance when the model confidently produces the same wrong SQL), we measure the syntactic diversity of the $M$ outputs via mean pairwise normalized edit distance:
\begin{align}
    d(x) \;=\; \frac{2}{M(M-1)}\sum_{i<j} \bigl(1 - \mathrm{sim}(\hat{s}_i,\hat{s}_j)\bigr),
    \quad\\
    \hat{\sigma}^2_c \;=\; \frac{1}{|\mathcal{R}_c|}\sum_{x\in\mathcal{R}_c} d(x),
    \quad
    \hat{\sigma}^2(x) \;=\; \max\!\bigl(\sigma^2_{\min},\;\lambda\cdot\hat{\sigma}^2_c\bigr),
\end{align}
where $\mathrm{sim}(\cdot,\cdot)$ is the longest-common-subsequence similarity ratio, $\mathcal{R}_c$ is the random sample from cluster $c$, $\sigma^2_{\min}=0.03$, and $\lambda=0.5$.  Syntactic diversity needs no ground truth and picks out the structurally uncertain queries: if the LLM writes the same SQL five times the query is unambiguous and $\hat{\sigma}^2$ is low; if it writes five different queries the annotation is likely to be volatile and $\hat{\sigma}^2$ is high.  Per-point diversities are averaged before pooling so that between-query variation in $f(x)$ does not inflate the cluster estimate.  Total cost: $C \times 5 \times M = 750$ LLM calls, run once before the active learning loop.

\paragraph{Assumed covariance $\Sigma_{\mathcal{C}\mathcal{C}}$.} In each round, Algorithm~\ref{alg:greedy} is run with $\Sigma_{\mathcal{C}\mathcal{C}}$ set to the GP posterior covariance of the pool given the examples already labeled, $K(C, C) - K(C, D_n)\bigl[K(D_n, D_n) + \sigma_0^2 I\bigr]^{-1} K(D_n, C)$ with $\sigma_0^2 = 0.05$. The results of Section~\ref{sec:ideal_theory} hold for any positive-definite $\Sigma_{\mathcal{C}\mathcal{C}}$, so conditioning on $D_n$ is permitted. Lemma~\ref{lemma:variance_mismatch} and Theorem~\ref{thm:robust_matroid} are stated for stationary prior kernels, but Lemma~\ref{lemma:variance_mismatch} holds for any conditioning set, and the posterior covariance given $D_n$ is the prior conditioned on $D_n$; both results therefore apply with $D_n$ added to the conditioning set throughout, which is how the diagnostics of Appendix~\ref{app:thm3_diag} are computed. The homoscedastic $\sigma_0^2$ on already-labeled points is an implementation simplification, and the heteroscedastic $\sigma^2(x)$ enters through $\Lambda$ for the candidates. We use an isotropic Mat\'ern-$1/2$ kernel (also known as the Exponential kernel, $k(r) = \sigma_f^2 \exp(-r/\ell)$) with length scale set by the median heuristic: $\ell = \sqrt{\operatorname{median}(\|x_i - x_j\|^2)/2}$ over the UMAP-30 representations of labeled and pool examples, recomputed each round. This avoids the degenerate solutions that arise from marginal-likelihood optimization with as few as 5--20 training points in 30 dimensions.

The choice of a rough kernel is tied to the direction of misspecification that Section~\ref{sec:misspecified_theory} covers. Against a Mat\'ern-$\nu^*$ truth with $\nu^* \ge 1/2$ and a comparable length scale, a Mat\'ern-$1/2$ surrogate satisfies Assumption~\ref{assum:spectral} with $s = \nu^* - 1/2 < \nu^*$, so Theorem~\ref{thm:robust_matroid} applies. A smoother surrogate than the truth is not covered, since its spectral density decays faster and $c_1 = 0$; neither is a squared-exponential truth, whose spectral density decays faster than any polynomial, so that the ratio $\beta(\omega)$ grows faster than $(1 + \|\omega\|^2)^s$ for every $s$. Mat\'ern-$1/2$ is therefore the choice that keeps the theorem applicable for the widest family of Mat\'ern truths, even though, as Section~\ref{sec:misspecified_theory} shows, the resulting constants are large on our pools.

\paragraph{Partition matroid.} In every experiment reported here, the clusters of the partition matroid are the seven query domains of the corpus (the analysis-type field stored with each example in the Milvus metadata), so that at most one query per domain is annotated in each round while at least $B$ domains remain. Each round of the annotation loop invokes Algorithm~\ref{alg:greedy} once, with batch size $B = 5$, on the current pool and its current partition. When the pool has at least $B$ domains, the guarantee of Theorem~\ref{thm:matroid_ideal} applies to each such invocation separately, and we do not claim a matroid guarantee for the union of selections across rounds. Two implementation details affect how the results should be read. The test split consumes every example of the two smallest domains (service and avc\_gpu\_analysis), so the pool contains at most five domains from the outset, and the four pool examples of pricing\_lt are exhausted after round 4, which leaves the enterprise pool with fewer than $B$ domains from round 5 onward. When fewer than $B$ clusters remain, the implementation fills the remaining slots of the batch by unconstrained greedy selection on $\delta_x$, and in the enterprise retrieval simulation this fallback was active in rounds 5 through 9. Batches that use the fallback lie outside $\mathcal{I}$, so Theorems~\ref{thm:matroid_ideal} and~\ref{thm:robust_matroid} do not apply to them. In the enterprise LLM evaluation (Section~\ref{sec:llm_eval}) and the LLM-scored-label experiment only two domains remain in the pool, so the fallback fills three of the five slots in every batch, and the two theorems cover no batch of those experiments; the same holds by design for the two-stratum fill-slot tests of Section~\ref{sec:strat_controls}. Spectral clustering on a $k$-nearest-neighbor graph in the projected space, the manifold-aware partition motivated in Section~\ref{sec:geometric_design}, is implemented as the fallback partition for corpora without domain labels and was not used in the reported runs.

\subsubsection{Experimental Setup}

\paragraph{Dataset.} The corpus comprises 477 domain-specific queries across seven domains, each paired with a ground-truth SQL written by database engineers. A stratified test set of 56 queries (10 each from supply\_gap and pbr; 8 each from hud, delivery, service, and pricing\_lt; all 4 avc\_gpu\_analysis queries) is held out and never added to the labeled bank. The remaining 421 queries form the $n_0 = 5$ seed and the unlabeled pool of $N = 416$. The retrieval simulation evaluates all 56 test queries; the end-to-end LLM evaluation uses the 20-query subset drawn from the two Databricks-dialect domains (supply\_gap and pbr), as cross-dialect SQL execution would confound scoring for the remaining five domains. Annotation begins from a biased seed of $n_0 = 5$ examples drawn exclusively from one domain, mirroring the real-world corpus imbalance and leaving all test queries outside of that domain with essentially zero retrieval score at round zero.

\paragraph{SQL Component Match Score.} Following the Component Match evaluation framework of \citet{yu2018spider}, our primary metric evaluates each SQL clause independently:
\begin{equation}
    \text{SemanticScore}(\hat{s}, s^*) = \frac{1}{|\mathcal{C}|} \sum_{c \in \mathcal{C}} M_c(\hat{s}, s^*)
\end{equation}
where $\mathcal{C} = \{\textsc{Select},\,\textsc{From},\,\textsc{Where},\,\textsc{Group By},\,\textsc{Order By},\,\textsc{Having}\}$ and each $M_c \in \{0,1\}$ is a binary match indicator. \textsc{From} uses exact table-set match (CTE-aware); \textsc{Group By} and \textsc{Order By} use exact column-set match (direction ignored); \textsc{Where} and \textsc{Having} use string-normalized exact match; \textsc{Select} requires complete column coverage ($\mathrm{CM}(\hat{s}, s^*) = 1$, taking the better of expression-based and alias-based strategies). This yields a score in $\{0,\tfrac{1}{6},\tfrac{1}{3},\tfrac{1}{2},\tfrac{2}{3},\tfrac{5}{6},1\}$. WHERE and HAVING use conservative string match following the Spider evaluation protocol. Two predicates that are logically equivalent but written differently will not match, so the reported scores understate true semantic correctness.

\paragraph{Evaluation protocols.} We run two evaluations, the second more realistic than the first:

\begin{enumerate}
    \item \textbf{Retrieval Simulation}: We measure two domain-discovery metrics at each annotation budget. (i)~\emph{Domains Covered}: the number of distinct test domains (out of 7) for which at least one test query achieves SemanticScore $\ge 0.5$ (majority of clauses match) against the labeled bank (Algorithm~\ref{alg:greedy}'s partition matroid should cover all domains fastest). (ii)~\emph{Non-Seed Domain Score}: mean SemanticScore restricted to the non-supply-gap test domains, measuring generalization beyond the cold-start seed bias. All metrics are oracle computations that isolate annotation selection from LLM generation quality.
    \item \textbf{End-to-End LLM Evaluation}: We use \texttt{meta/llama-3.1-70b-instruct} (NVIDIA NIM) to generate SQL from the nearest labeled example as the few-shot context and score with SemanticScore. This is the downstream quantity we actually care about.
\end{enumerate}

\subsubsection{Baselines}

We compare against five baselines. Three are unstratified: a no-information control, a classical label-driven active learner, and a purely geometric space-filling rule. The other two are stratified controls that share SHARP's partition and fill rule and differ only in how the query inside each stratum is chosen, so whatever separates them from SHARP is attributable to the GP score and not to the constraint.

\begin{itemize}
    \item \textbf{Random}: Uniform random selection from the unlabeled pool.
    \item \textbf{Uncertainty Sampling} \citep{lewis1994sequential}: Logistic regression trained on binarized oracle labels; selects queries with maximum classification uncertainty.
    \item \textbf{Distance-to-Threshold}: Selects the $K$ pool queries with the largest minimum L2 distance to any currently labeled example. This is the greedy max-min-distance (space-filling, or $k$-center) strategy, which is known to be near-optimal for covering the embedding manifold.
    \item \textbf{Stratified Random}: In each batch, one query is drawn uniformly at random from each of $B$ randomly ordered strata, using the partition of Algorithm~\ref{alg:greedy}; when fewer than $B$ strata remain, the remaining slots are filled by uniform random selection from the whole pool, mirroring SHARP's unconstrained fill.
    \item \textbf{Stratified Distance-to-Threshold}: The same partition and fill rule, with the query in each stratum chosen by the max-min-distance rule of Distance-to-Threshold. This is SHARP's greedy step with the posterior variance replaced by plain distance and the noise weighting removed, that is, the farthest-point rule to which Lemma~\ref{lemma:one_neighbor} relates the GP score.
\end{itemize}

The stratified controls were added after the original evaluation, when the ablation of Section~\ref{sec:ablation} had attributed the gains to the partition constraint. Their retrieval simulations use the seeds and splits of the other methods. In the end-to-end LLM evaluations, Stratified Random is run through the same code path as SHARP, with the same retrieval depth, prompt, splits, and seed bank, so that only selection differs; because it was run about three weeks after the SHARP runs, the table reports the difference at $n = 5$, where both arms use the identical seed bank, as a measure of endpoint drift.

\subsubsection{Intrinsic Dimension Estimation}
\label{sec:experiments_dim}

The geometric design of Section~\ref{sec:geometric_design}, and the choice to fit the GP and cluster the pool in a 30-dimensional projection, rest on the assumption that queries lie on a low-dimensional intrinsic manifold $\mathcal{M} \subset \mathbb{R}^d$. We estimate $d_I$ using the Levina--Bickel MLE \citep{levina2004maximum} with $k = 10$ neighbors on 200 training embeddings:
\begin{equation}
    \hat{d}_I = \left( \frac{1}{n} \sum_{i=1}^{n} \frac{1}{k-1} \sum_{j=1}^{k-1} \log \frac{r_k(x_i)}{r_j(x_i)} \right)^{-1}
\end{equation}
where $r_j(x_i)$ is the distance from $x_i$ to its $j$-th nearest neighbor. We get $\hat{d}_I \approx 18.4$, far below $d = 2048$, which supports the manifold assumption and motivates the choice of 30 projected dimensions: 30 is comfortably above 18.4 and brings the GP down from 2048 dimensions to something tractable. The estimate alone does not show that the projection preserves the relevant geometry; the evidence for that is the ablation of Section~\ref{sec:ablation}, in which UMAP and PCA projections agree to within noise.

\subsubsection{Results: Retrieval Simulation}

The retrieval simulation evaluates how quickly each method discovers the full breadth of query domains. We report Domains Covered (out of 7) and the Non-Seed Domain Score (mean SemanticScore on the six test domains outside of the seed domain). All numbers are means over 10 random seeds.

\begin{table}[h]
\centering
\caption{Domain discovery on the retrieval oracle simulation. Domains Covered cells show mean over 10 seeds; subscript $\pm\sigma$ is the s.d.\ across seeds (omitted where zero). \emph{Domains Covered}: number of test domains (out of 7) with $\ge$1 query at SemanticScore $\ge 0.5$. \emph{To 6 dom.}: annotation budget at which mean 6-domain coverage is first achieved (no method reaches all 7). Because a domain counts as covered only when a retrieved example yields a score of at least $0.5$, coverage can fall as the bank grows; Uncertainty Sampling reaches $6.0$ at $n = 45$ and drops back to $5.9$ at $n = 50$. CovAUC = area under Domains-Covered curve / total budget. The stratified controls (below the rule) share SHARP's partition and differ only in the within-stratum choice. Best per column in \textbf{bold}.}
\label{tab:sim_results}
\resizebox{\linewidth}{!}{%
\begin{tabular}{lccccccc}
\toprule
\textbf{Method} & $n{=}10$ & $n{=}20$ & $n{=}30$ & $n{=}40$ & \textbf{Final} & \textbf{To 6 dom.} & \textbf{CovAUC} \\
\midrule
Uncertainty Sampling  & $5.3_{\pm0.5}$ & $5.6_{\pm0.5}$ & $5.8_{\pm0.4}$ & $5.9_{\pm0.3}$ & $5.9_{\pm0.3}$ & 45          & 5.08 \\
Random                & $5.2_{\pm0.4}$ & $5.5_{\pm0.5}$ & $5.7_{\pm0.5}$ & $5.9_{\pm0.3}$ & 6.0            & 50          & 5.03 \\
Dist-to-Threshold     & $5.4_{\pm0.5}$ & 6.0            & 6.0            & 6.0            & 6.0            & 15          & 5.28 \\
\midrule
Stratified Random     & \textbf{6.0}   & \textbf{6.0}   & \textbf{6.0}   & \textbf{6.0}   & \textbf{6.0}   & \textbf{10} & \textbf{5.34} \\
Stratified DtT        & \textbf{6.0}   & \textbf{6.0}   & \textbf{6.0}   & \textbf{6.0}   & \textbf{6.0}   & \textbf{10} & \textbf{5.34} \\
\textbf{SHARP}        & \textbf{6.0}   & \textbf{6.0}   & \textbf{6.0}   & \textbf{6.0}   & \textbf{6.0}   & \textbf{10} & \textbf{5.34} \\
\bottomrule
\end{tabular}}
\end{table}

\begin{figure}[h]
\centering
\includegraphics[width=0.60\textwidth]{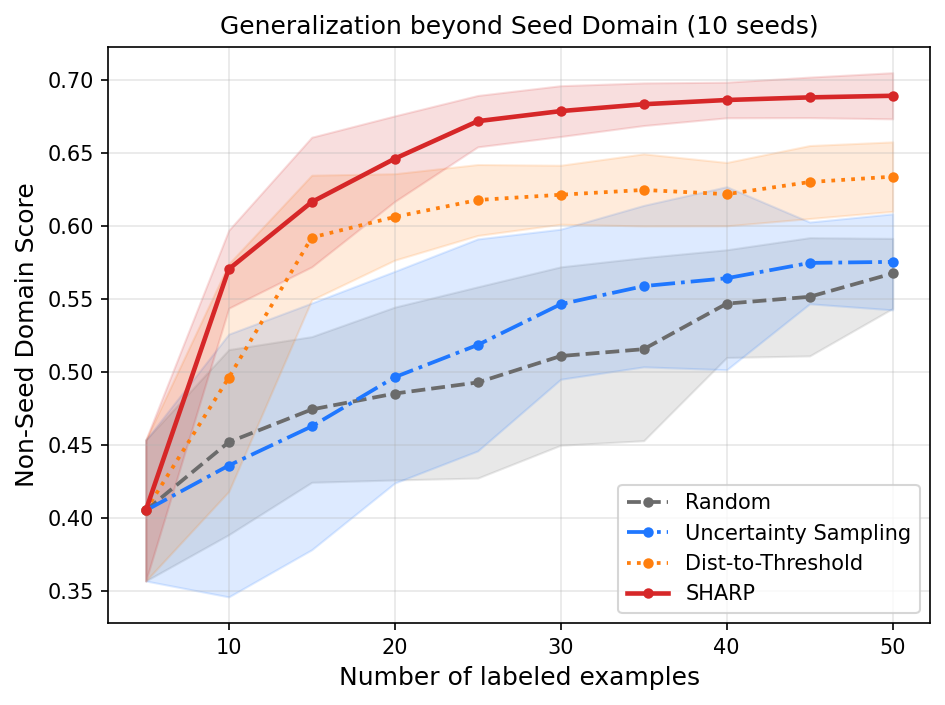}
\caption{Non-Seed Domain Score learning curves (mean $\pm1$ s.d.\ over 10 seeds) for SHARP and the unstratified baselines. All methods start from an identical supply-gap-only seed ($n_0 = 5$). Algorithm~\ref{alg:greedy} leads the unstratified baselines at every budget level, finishing at $0.690$ vs.\ $0.634$ for Dist-to-Threshold. The stratified controls, not plotted, track SHARP within one s.d.\ throughout and finish at $0.690$ (Stratified Random) and $0.687$ (Stratified DtT); see the text.}
\label{fig:learning_curves_sim}
\end{figure}

\begin{figure}[h]
\centering
\includegraphics[width=0.7\textwidth]{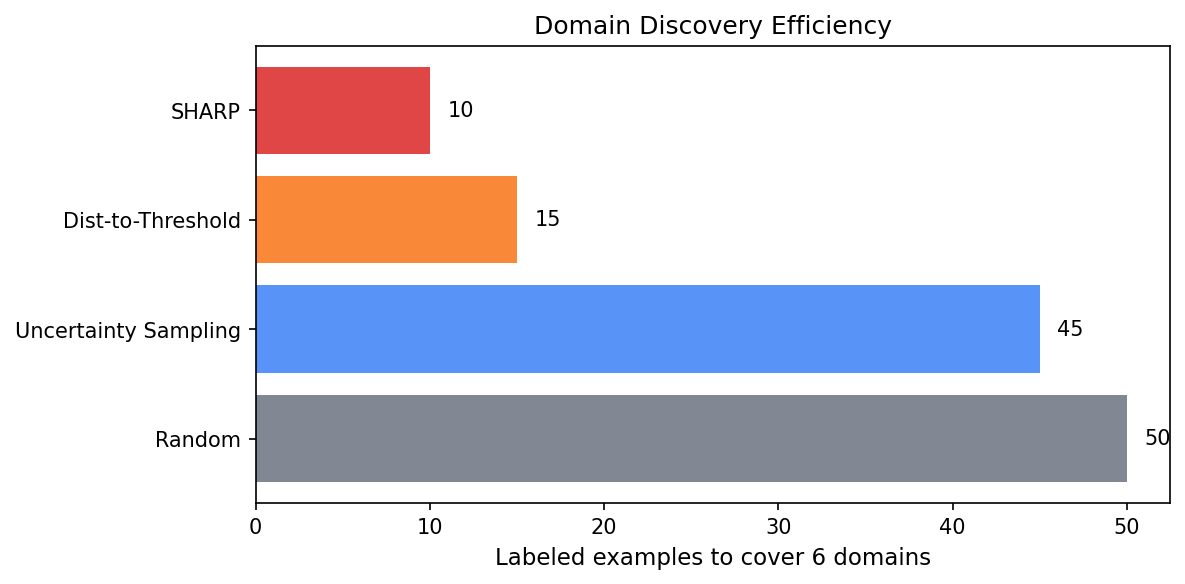}
\caption{Annotation efficiency: labeled examples needed to first reach 6-domain coverage (of 7 domains), for SHARP and the unstratified baselines. Algorithm~\ref{alg:greedy} reaches 6-domain coverage at $n = 10$, as do the two stratified controls (not plotted; their curves coincide with SHARP's); Dist-to-Threshold reaches 6 domains at $n = 15$; Random requires the full $n = 50$ budget to reach 6 domains. No method achieves full 7-domain coverage within the $n = 50$ budget.}
\label{fig:efficiency_sim}
\end{figure}

Table~\ref{tab:sim_results} summarizes the headline numbers, and Figures~\ref{fig:learning_curves_sim} and~\ref{fig:efficiency_sim} plot the corresponding learning curves and annotation-efficiency comparison.

Against the unstratified methods, Algorithm~\ref{alg:greedy} covers the domains earliest and most broadly. It reaches 6/7 domains at $n = 10$ (zero variance across seeds), while Dist-to-Threshold does not reach 6 domains until $n = 15$ and Random requires the full $n = 50$ budget to do so. Algorithm~\ref{alg:greedy} also achieves a higher Coverage AUC ($5.34$ vs.\ $5.28$ for Dist-to-Threshold and $5.03$ for Random).

The partition matroid constraint front-loads domain discovery. At $n = 10$, Algorithm~\ref{alg:greedy} has already covered $6.0$ domains (all seeds) vs.\ $5.4$ for Dist-to-Threshold and $5.2$ for Random. The reason is simply that the constraint forces one example per cluster into every batch, so even from 5 seed examples the next batch spans several query types. The two stratified controls, which are forced in the same way, have coverage curves identical to SHARP's in every seed. On the Non-Seed Domain Score, Algorithm~\ref{alg:greedy} leads the unstratified baselines at every budget level (Non-Seed $= 0.571$ at $n = 10$ vs.\ $0.496$ for Dist-to-Threshold and $0.452$ for Random) and maintains the lead to the final budget ($0.690$ vs.\ $0.634$ for Dist-to-Threshold). It does not lead the stratified controls. Stratified Random finishes at $0.690$ and Stratified Dist-to-Threshold at $0.687$ (paired $t$-tests over seeds against SHARP, $p = 0.81$ and $p = 0.57$), and at $n = 10$ both are slightly ahead of SHARP ($0.590$ and $0.582$ vs.\ $0.571$). In the simulation, then, the whole advantage over the unstratified baselines is delivered by the partition, and the GP score chooses within strata no better than a uniform draw.

\subsubsection{Results: End-to-End LLM Evaluation}
\label{sec:llm_eval}

The LLM evaluation is the test that matters, since it asks whether the selected examples actually improve the generated SQL. For each test query, the LLM (\texttt{meta/llama-3.1-70b-instruct}) generates SQL from the retrieved few-shot context and the output is scored with SemanticScore. We drop Uncertainty Sampling here. It had the weakest domain discovery in the retrieval simulation (Table~\ref{tab:sim_results}): it is the only method that does not reliably reach 6/7 domains within the $n = 50$ budget (mean coverage $5.9/7$ at $n = 50$) and it has the lowest Coverage AUC ($5.08$) of the three unstratified baselines. LLM calls are the expensive part of the evaluation, so we spend them on Random, on Distance-to-Threshold as the strongest geometric baseline, on Algorithm~\ref{alg:greedy}, and on the Stratified Random control. Stratified Distance-to-Threshold enters the LLM evaluation only in the pre-specified tests of Section~\ref{sec:strat_controls}.

\begin{table}[h]
\centering
\caption{End-to-end LLM SemanticScore and Non-Seed Domain Score at each labeled-set size (mean $\pm$ 1 bootstrap SE, 10 seeds). Strat-R is Stratified Random, which shares SHARP's partition, retrieval depth, prompt, splits, and seed bank; at $n = 5$ its bank is identical to SHARP's, so the $0.007$ difference in SemanticScore there is endpoint drift between the two runs. Best per row and metric in \textbf{bold}. Against the unstratified baselines, Algorithm~\ref{alg:greedy} trails at $n \le 10$ and leads from $n = 15$ onward. Against Stratified Random it is ahead at $n = 10$ and $15$ and tied at $n = 50$ ($0.543$ vs.\ $0.545$).}
\label{tab:llm_results}
\small
\resizebox{\linewidth}{!}{%
\begin{tabular}{lcccccccc}
\toprule
& \multicolumn{4}{c}{\textbf{SemanticScore (all domains)}} & \multicolumn{4}{c}{\textbf{Non-Seed Domain Score}} \\
\cmidrule(lr){2-5}\cmidrule(lr){6-9}
\textbf{$n$} & \textbf{Random} & \textbf{DtT} & \textbf{Strat-R} & \textbf{SHARP} & \textbf{Random} & \textbf{DtT} & \textbf{Strat-R} & \textbf{SHARP} \\
\midrule
5  & $0.449_{\pm.014}$ & $\mathbf{0.450_{\pm.013}}$ & $0.410_{\pm.010}$ & $0.403_{\pm.010}$ & $0.453_{\pm.014}$ & $\mathbf{0.460_{\pm.016}}$ & $0.395_{\pm.011}$ & $0.395_{\pm.011}$ \\
10 & $0.468_{\pm.015}$ & $\mathbf{0.478_{\pm.016}}$ & $0.433_{\pm.012}$ & $0.471_{\pm.014}$ & $0.460_{\pm.015}$ & $0.485_{\pm.017}$ & $0.460_{\pm.017}$ & $\mathbf{0.507_{\pm.019}}$ \\
15 & $0.471_{\pm.015}$ & $0.487_{\pm.016}$ & $0.454_{\pm.013}$ & $\mathbf{0.498_{\pm.016}}$ & $0.458_{\pm.017}$ & $0.492_{\pm.018}$ & $0.497_{\pm.019}$ & $\mathbf{0.560_{\pm.023}}$ \\
20 & $0.483_{\pm.017}$ & $0.492_{\pm.016}$ & $0.486_{\pm.016}$ & $\mathbf{0.506_{\pm.017}}$ & $0.478_{\pm.019}$ & $0.495_{\pm.021}$ & $0.525_{\pm.023}$ & $\mathbf{0.585_{\pm.025}}$ \\
25 & $0.510_{\pm.016}$ & $0.506_{\pm.016}$ & $0.498_{\pm.016}$ & $\mathbf{0.519_{\pm.017}}$ & $0.485_{\pm.019}$ & $0.517_{\pm.021}$ & $0.550_{\pm.022}$ & $\mathbf{0.593_{\pm.025}}$ \\
30 & $0.506_{\pm.016}$ & $0.507_{\pm.016}$ & $0.513_{\pm.017}$ & $\mathbf{0.522_{\pm.017}}$ & $0.480_{\pm.019}$ & $0.515_{\pm.020}$ & $0.563_{\pm.024}$ & $\mathbf{0.582_{\pm.025}}$ \\
35 & $0.513_{\pm.016}$ & $0.506_{\pm.017}$ & $0.515_{\pm.017}$ & $\mathbf{0.525_{\pm.017}}$ & $0.487_{\pm.019}$ & $0.512_{\pm.022}$ & $0.568_{\pm.024}$ & $\mathbf{0.585_{\pm.025}}$ \\
40 & $0.508_{\pm.017}$ & $0.509_{\pm.017}$ & $0.536_{\pm.017}$ & $\mathbf{0.539_{\pm.018}}$ & $0.473_{\pm.018}$ & $0.512_{\pm.023}$ & $0.570_{\pm.024}$ & $\mathbf{0.613_{\pm.026}}$ \\
45 & $0.514_{\pm.017}$ & $0.521_{\pm.018}$ & $0.546_{\pm.017}$ & $\mathbf{0.557_{\pm.019}}$ & $0.487_{\pm.018}$ & $0.517_{\pm.022}$ & $0.588_{\pm.025}$ & $\mathbf{0.617_{\pm.026}}$ \\
50 & $0.514_{\pm.017}$ & $0.532_{\pm.018}$ & $\mathbf{0.545_{\pm.018}}$ & $0.543_{\pm.018}$ & $0.488_{\pm.020}$ & $0.528_{\pm.022}$ & $0.593_{\pm.025}$ & $\mathbf{0.608_{\pm.025}}$ \\
\bottomrule
\end{tabular}}
\end{table}

\begin{table}[h]
\centering
\caption{Table Match Rate and Column Coverage at $n=50$ labeled examples. Values are mean $\pm$ 1 bootstrap SE over 10 seeds (200 total test queries; 20 per seed).}
\label{tab:llm_breakdown}
\begin{tabular}{lcccc}
\toprule
\textbf{Metric} & \textbf{Random} & \textbf{Dist-to-Threshold} & \textbf{Strat.\ Random} & \textbf{SHARP} \\
\midrule
Table Match Rate  & $0.365_{\pm0.034}$ & $0.435_{\pm0.036}$ & $\bm{0.470_{\pm0.036}}$ & $0.455_{\pm0.035}$ \\
Column Coverage   & $0.369_{\pm0.028}$ & $0.400_{\pm0.029}$ & $0.566_{\pm0.029}$ & $\bm{0.578_{\pm0.027}}$ \\
SemanticScore     & $0.514_{\pm0.017}$ & $0.532_{\pm0.018}$ & $\bm{0.545_{\pm0.018}}$ & $0.543_{\pm0.018}$ \\
\bottomrule
\end{tabular}
\end{table}

\begin{figure}[h]
\centering
\includegraphics[width=0.95\textwidth]{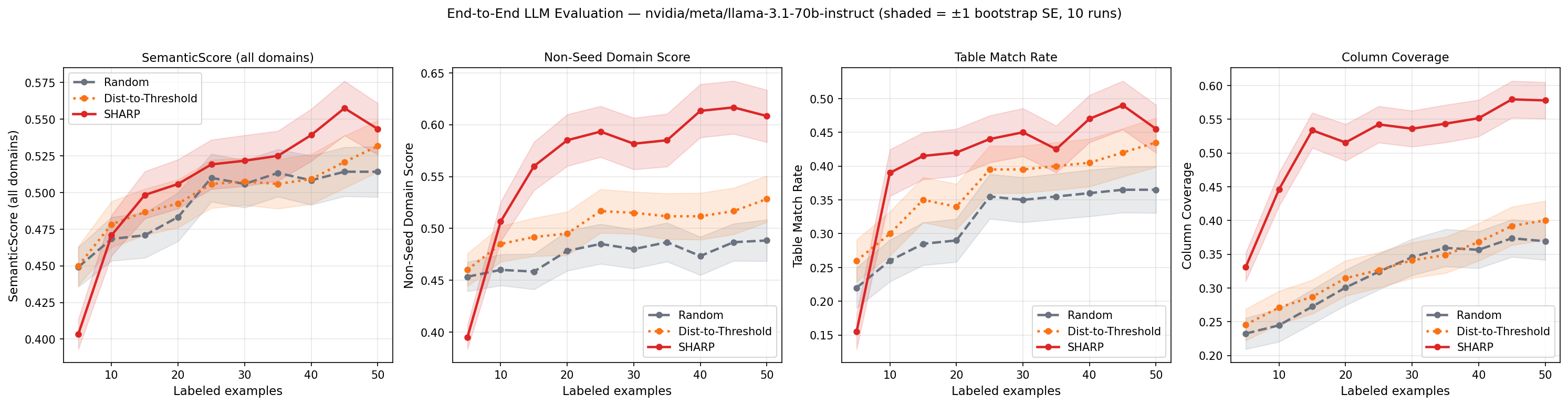}
\caption{End-to-end LLM evaluation vs.\ annotation budget for SHARP and the unstratified baselines; panels from left to right show SemanticScore, Non-Seed Domain Score, Table Match Rate, and Column Coverage (shaded bands: $\pm$1 bootstrap SE over 10 seeds; Stratified Random is reported in Table~\ref{tab:llm_results}). Algorithm~\ref{alg:greedy} trails at $n \le 10$ on overall SemanticScore but leads from $n = 15$ onward, finishing at $0.543$ vs.\ $0.532$ for Dist-to-Threshold and $0.514$ for Random. The Non-Seed Domain Score (second panel) shows Algorithm~\ref{alg:greedy}'s cross-domain advantage emerging from $n = 10$, peaking at $0.617$ at $n = 45$, and finishing at $0.608$ vs.\ $0.528$ for Dist-to-Threshold.}
\label{fig:learning_curves_llm}
\end{figure}

Table~\ref{tab:llm_results} reports SemanticScore and Non-Seed Domain Score at each labeled-set size, and Figure~\ref{fig:learning_curves_llm} plots the corresponding learning curves.

The LLM evaluation shows a budget-dependent trajectory for Algorithm~\ref{alg:greedy}. At small budgets ($n \le 10$), baselines match or exceed Algorithm~\ref{alg:greedy} on overall SemanticScore. The Non-Seed Domain Score separates earlier: at $n = 10$, Algorithm~\ref{alg:greedy} achieves $0.507$ on non-seed-domain queries vs.\ $0.485$ for Dist-to-Threshold and $0.460$ for Random; at $n = 15$ the gap is already $0.560$ vs.\ $0.492$, and Algorithm~\ref{alg:greedy} peaks at $0.617$ at $n = 45$ vs.\ $0.517$ for Dist-to-Threshold, a $1.2{\times}$ ratio that comes from the partition's cross-domain selection.

Algorithm~\ref{alg:greedy} leads on overall SemanticScore from $n = 15$ onward, finishing at $0.543$ vs.\ $0.532$ for Dist-to-Threshold and $0.514$ for Random; the final margin over Dist-to-Threshold is within one bootstrap SE (paired Wilcoxon over the 200 pooled test queries: $p = 0.16$ vs.\ Dist-to-Threshold, $p = 0.04$ vs.\ Random), while the margin at $n = 45$ ($0.557$ vs.\ $0.521$, $p = 0.02$) exceeds one bootstrap SE. The breakdown at $n = 50$ (Table~\ref{tab:llm_breakdown}) locates the advantage: Column Coverage is $0.578$ vs.\ $0.400$ for Dist-to-Threshold ($p < 0.001$), whereas Table Match Rate ($0.455$ vs.\ $0.435$) is indistinguishable. So the bank chosen by Algorithm~\ref{alg:greedy} is no better at finding the right tables, but once the table is found it is considerably better at surfacing the right columns. The Non-Seed Domain Score at $n = 50$ is $0.608$ for Algorithm~\ref{alg:greedy} vs.\ $0.528$ for Dist-to-Threshold and $0.488$ for Random ($p = 0.03$), which is what broader semantic coverage of the bank should look like.

Stratified Random shows how much of this is due to the partition. At $n = 50$ it matches Algorithm~\ref{alg:greedy} on every metric in Table~\ref{tab:llm_breakdown}: SemanticScore $0.545$ vs.\ $0.543$ ($p = 0.99$), Column Coverage $0.566$ vs.\ $0.578$ ($p = 0.71$), Table Match Rate $0.470$ vs.\ $0.455$ ($p = 0.71$), and Non-Seed Domain Score $0.593$ vs.\ $0.608$ ($p = 0.51$). The Column Coverage advantage over Dist-to-Threshold is therefore a property of any bank that samples every domain, not of the GP score in particular. The one place where the two stratified arms separate is at small budgets. Algorithm~\ref{alg:greedy} leads Stratified Random at $n = 10$ ($+0.038$, $p = 0.005$) and $n = 15$ ($+0.044$, $p = 0.003$), and the SemanticScore AUC over the full budget differs by $+0.015$ ($p = 0.05$, paired $t$-test over seeds). These tests were not specified in advance, so we treated them as generating a hypothesis and tested it on fresh seeds. As Section~\ref{sec:strat_controls} reports, the small-budget advantage does not survive correction for multiple testing, and where it appears it is matched by the farthest-point rule, so it does not require the GP.

\subsubsection{Results: LLM-Scored Labels without Oracle Access}
\label{sec:realistic}

The previous two evaluations use oracle GP labels derived from structural SQL similarity to the test set.  This section replaces the oracle with the same labeling procedure used in production: for each annotated example, the LLM is called to generate SQL, and the output is scored against the ground-truth SQL using the production formula
\begin{equation}
    f(x) =
    \begin{cases}
        0.3 + 0.7 \cdot \text{ColCov}(x) & \text{if } \text{TableMatch}(x) = 1, \\
        0.2 \cdot \text{ColCov}(x)        & \text{otherwise,}
    \end{cases}
    \label{eq:prod_score}
\end{equation}
where $\text{TableMatch}(x) \in \{0,1\}$ indicates whether the generated SQL references the correct tables and $\text{ColCov}(x) \in [0,1]$ measures the fraction of ground-truth columns recovered.  This formula gives partial credit for correct-table, partial-column matches, producing richer variance for GP training than the binary Spider component score. The production formula in Equation~\ref{eq:prod_score} is used only to compute GP training labels. The reported SemanticScore in Tables~\ref{tab:realistic_results} and~\ref{tab:realistic_breakdown} remains the Component Match Score from Section~\ref{sec:experiments} (mean over the six SQL clauses), to keep the metric comparable to the oracle LLM evaluation in Section~\ref{sec:llm_eval}.

\paragraph{Setup.}
The experiment follows the same train/test split as the LLM evaluation (Section~\ref{sec:llm_eval}), restricting to only two domains.  The annotation loop proceeds as follows:

\begin{enumerate}
    \item \textbf{Round 0 (cold start).}  The $n_0 = 5$ seed examples are labeled with empty few-shot context: the LLM receives only the DDL (Data Definition Language) schema and the question.  The resulting scores $\{f(x_i)\}_{i=1}^{n_0}$ initialize the GP.
    \item \textbf{Subsequent rounds.}  At each round, the acquisition strategy selects a batch of $B = 5$ candidates from the remaining pool.  Each selected example is scored by calling the LLM with the $k = 5$ nearest neighbors from the current labeled set as few-shot context. The DDL schema is prepended to every prompt, matching the production prompt template.  The GP posterior mean is refitted on the updated label set after each round.
\end{enumerate}

The GP uses the same Matérn-$1/2$ kernel, median-heuristic length scale, and self-consistency noise prior as in Section~\ref{sec:llm_eval}.  Embeddings are reduced to 30 dimensions with the same UMAP projection as in Section~\ref{sec:llm_eval}.  Three methods are compared in the main run: \textbf{Random} (uniform pool sampling; LLM labels collected but not used for selection), \textbf{Dist-to-Threshold} (label-free geometric heuristic), and \textbf{SHARP} (Algorithm~\ref{alg:greedy}); a fourth arm, \textbf{Stratified Random} (Section~\ref{sec:strat_controls}), was run separately with the same splits and seed banks.  The experiment is run over 3 independent random seeds; all reported values are mean $\pm$ 1 SE across the three seeds.

\paragraph{What this experiment does and does not test.}
In the oracle LLM experiment (Section~\ref{sec:llm_eval}), the GP labels are computed once upfront as $\max_{t \in \mathcal{T}} \text{SemanticScore}(s_i, t)$, where $\mathcal{T}$ is the held-out test SQL set; this oracle requires knowledge of the test set.  Here, labels are computed online from LLM generation, so no test SQL is observed at any point during the annotation loop, which mirrors the production deployment in which the budget is spent without access to a held-out evaluation set.  It is worth being careful about what does and does not change as a result.  By Remark~\ref{rem:label_dependence}, Algorithm~\ref{alg:greedy} uses labels only through the fitted hyperparameters and noise estimates, and in the configuration reported here both are label-free (median-heuristic length scale, self-consistency noise prior).  The selected sets therefore do not depend on the LLM-scored labels, and for a given seed they coincide with what the oracle configuration would have selected.  The experiment tests whether the selection advantage of Sections~\ref{sec:experiments} and~\ref{sec:llm_eval} survives when the entire pipeline, including the few-shot context used to label each new example and the DDL-augmented prompt, is run as in production and evaluated with no oracle anywhere in the loop.  It does not test label-driven acquisition, which would require a configuration in which $\hat{\theta}$ or $\hat{\Lambda}$ are refit from the accumulating labels; we leave that to future work.

\begin{figure}[h]
\centering
\includegraphics[width=0.95\textwidth]{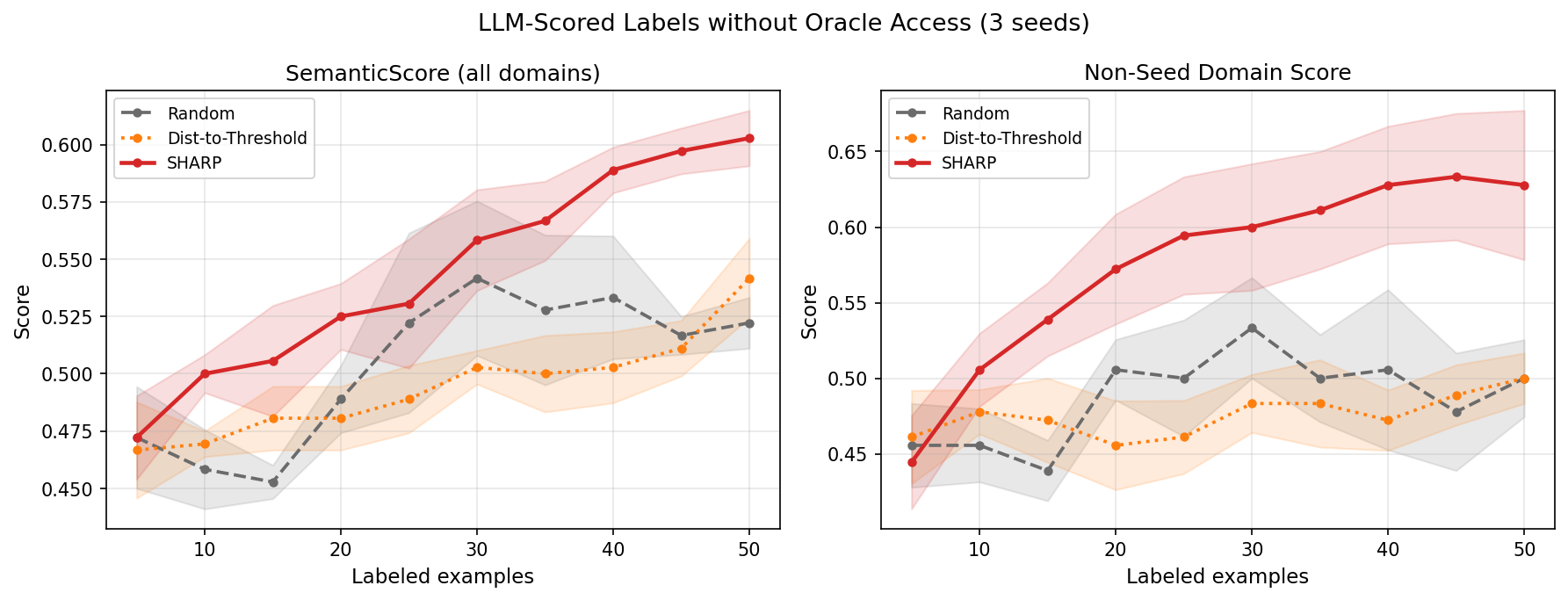}
\caption{LLM-scored-label experiment: SemanticScore (left) and Non-Seed Domain Score (right) vs.\ annotation budget (mean $\pm$ 1 SE over 3 seeds).  Labels are computed by calling the LLM on each annotated example and scoring with the production formula (Equation~\ref{eq:prod_score}); no oracle test SQL is observed during the loop.  Selection under Algorithm~\ref{alg:greedy} is label-free in this configuration (Remark~\ref{rem:label_dependence}); Random and Dist-to-Threshold serve as baselines.  The DDL schema is included in all prompts.}
\label{fig:learning_curves_realistic}
\end{figure}

\begin{table}[h]
\centering
\caption{LLM-scored-label experiment: SemanticScore and Non-Seed Domain Score at each labeled-set size (mean $\pm$ 1 SE over 3 seeds).  Labels come from LLM scoring (Equation~\ref{eq:prod_score}); DDL included in all prompts. Strat-R (Stratified Random) was run separately with the same splits, seed banks, retrieval depth, and prompt as SHARP; at $n = 5$ its bank is identical to SHARP's. Best per row and metric in \textbf{bold}.}
\label{tab:realistic_results}
\small
\resizebox{\linewidth}{!}{%
\begin{tabular}{lcccccccc}
\toprule
& \multicolumn{4}{c}{\textbf{SemanticScore}} & \multicolumn{4}{c}{\textbf{Non-Seed Domain Score}} \\
\cmidrule(lr){2-5}\cmidrule(lr){6-9}
\textbf{$n$} & \textbf{Random} & \textbf{DtT} & \textbf{Strat-R} & \textbf{SHARP} & \textbf{Random} & \textbf{DtT} & \textbf{Strat-R} & \textbf{SHARP} \\
\midrule
5  & $\mathbf{0.472_{\pm.022}}$ & $0.467_{\pm.021}$ & $0.467_{\pm.021}$ & $\mathbf{0.472_{\pm.018}}$ & $0.456_{\pm.028}$ & $\mathbf{0.461_{\pm.031}}$ & $0.444_{\pm.031}$ & $0.444_{\pm.031}$ \\
10 & $0.458_{\pm.017}$ & $0.469_{\pm.006}$ & $\mathbf{0.500_{\pm.013}}$ & $\mathbf{0.500_{\pm.008}}$ & $0.456_{\pm.024}$ & $0.478_{\pm.015}$ & $0.478_{\pm.015}$ & $\mathbf{0.506_{\pm.024}}$ \\
15 & $0.453_{\pm.007}$ & $0.481_{\pm.014}$ & $\mathbf{0.528_{\pm.028}}$ & $0.506_{\pm.024}$ & $0.439_{\pm.020}$ & $0.472_{\pm.028}$ & $0.528_{\pm.020}$ & $\mathbf{0.539_{\pm.024}}$ \\
20 & $0.489_{\pm.015}$ & $0.481_{\pm.014}$ & $\mathbf{0.531_{\pm.026}}$ & $0.525_{\pm.014}$ & $0.506_{\pm.020}$ & $0.456_{\pm.029}$ & $0.528_{\pm.020}$ & $\mathbf{0.572_{\pm.036}}$ \\
25 & $0.522_{\pm.039}$ & $0.489_{\pm.015}$ & $\mathbf{0.533_{\pm.025}}$ & $0.531_{\pm.028}$ & $0.500_{\pm.039}$ & $0.461_{\pm.024}$ & $0.528_{\pm.039}$ & $\mathbf{0.594_{\pm.039}}$ \\
30 & $0.542_{\pm.034}$ & $0.503_{\pm.007}$ & $0.533_{\pm.029}$ & $\mathbf{0.558_{\pm.022}}$ & $0.533_{\pm.033}$ & $0.483_{\pm.019}$ & $0.544_{\pm.036}$ & $\mathbf{0.600_{\pm.042}}$ \\
35 & $0.528_{\pm.033}$ & $0.500_{\pm.017}$ & $0.539_{\pm.043}$ & $\mathbf{0.567_{\pm.017}}$ & $0.500_{\pm.029}$ & $0.483_{\pm.029}$ & $0.567_{\pm.067}$ & $\mathbf{0.611_{\pm.039}}$ \\
40 & $0.533_{\pm.027}$ & $0.503_{\pm.016}$ & $0.539_{\pm.031}$ & $\mathbf{0.589_{\pm.010}}$ & $0.506_{\pm.053}$ & $0.472_{\pm.020}$ & $0.550_{\pm.059}$ & $\mathbf{0.628_{\pm.039}}$ \\
45 & $0.517_{\pm.008}$ & $0.511_{\pm.012}$ & $0.558_{\pm.021}$ & $\mathbf{0.597_{\pm.010}}$ & $0.478_{\pm.039}$ & $0.489_{\pm.020}$ & $0.583_{\pm.042}$ & $\mathbf{0.633_{\pm.042}}$ \\
50 & $0.522_{\pm.011}$ & $0.542_{\pm.017}$ & $0.569_{\pm.017}$ & $\mathbf{0.603_{\pm.012}}$ & $0.500_{\pm.026}$ & $0.500_{\pm.017}$ & $0.567_{\pm.019}$ & $\mathbf{0.628_{\pm.049}}$ \\
\bottomrule
\end{tabular}}
\end{table}

\begin{table}[h]
\centering
\caption{Structural breakdown at $n=50$ for the LLM-scored-label experiment (mean $\pm$ 1 SE over 3 seeds).}
\label{tab:realistic_breakdown}
\begin{tabular}{lcccc}
\toprule
\textbf{Metric} & \textbf{Random} & \textbf{Dist-to-Threshold} & \textbf{Strat.\ Random} & \textbf{SHARP} \\
\midrule
SemanticScore     & $0.522_{\pm.011}$ & $0.542_{\pm.017}$ & $0.569_{\pm.017}$ & $\mathbf{0.603_{\pm.012}}$ \\
Table Match Rate  & $0.433_{\pm.017}$ & $0.433_{\pm.017}$ & $\mathbf{0.533_{\pm.033}}$ & $\mathbf{0.533_{\pm.060}}$ \\
Column Coverage   & $0.352_{\pm.016}$ & $0.359_{\pm.003}$ & $0.457_{\pm.045}$ & $\mathbf{0.509_{\pm.007}}$ \\
\bottomrule
\end{tabular}
\end{table}

Table~\ref{tab:realistic_results} and Figure~\ref{fig:learning_curves_realistic} report the results.

Algorithm~\ref{alg:greedy} leads on SemanticScore from $n = 10$ onward, finishing at $0.603 \pm 0.012$ vs.\ $0.542 \pm 0.017$ for Dist-to-Threshold and $0.522 \pm 0.011$ for Random.  The Non-Seed Domain Score gap is consistent: $0.628 \pm 0.049$ vs.\ $0.500 \pm 0.017$ for Dist-to-Threshold and $0.500 \pm 0.026$ for Random at $n = 50$, so the cross-domain coverage produced by the partition matroid carries through to a pipeline with no oracle in it.  The structural breakdown at $n = 50$ (Table~\ref{tab:realistic_breakdown}) shows Algorithm~\ref{alg:greedy} leading on all metrics: Table Match Rate ($0.533$ vs.\ $0.433$ for both baselines) and Column Coverage ($0.509$ vs.\ $0.359$ for Dist-to-Threshold and $0.352$ for Random).

The Non-Seed Domain Score for Algorithm~\ref{alg:greedy} rises from $0.506$ at $n = 10$ to $0.633$ at $n = 45$, with a single dip to $0.628$ at $n = 50$ that is within one SE, while both baselines plateau near $0.5$. The advantage accumulates steadily with budget.

The learning curves are noisier than in the oracle evaluation. Selection is label-free here, so the noise is not coming from the acquisition path; it comes from the evaluation, where real LLM scores fluctuate (an ambiguous query produces different SQL on different calls) and only 3 seeds are averaged. The SE bands in Figure~\ref{fig:learning_curves_realistic} show as much. Still, the trend is upward and neither unstratified baseline follows Algorithm~\ref{alg:greedy}'s trajectory.

Stratified Random, run with the same three splits and seed banks, matches or leads Algorithm~\ref{alg:greedy} on SemanticScore up to $n = 25$ ($0.528$ vs.\ $0.506$ at $n = 15$) but trails it on Non-Seed Domain Score at every budget after $n = 5$, and from $n = 30$ it falls behind on both. At $n = 50$ it reaches a SemanticScore of $0.569 \pm 0.017$ against $0.603 \pm 0.012$, a Non-Seed Domain Score of $0.567 \pm 0.019$ against $0.628 \pm 0.049$, and a Column Coverage of $0.457 \pm 0.045$ against $0.509 \pm 0.007$, with the same Table Match Rate ($0.533$). This is the only evaluation in the paper in which SHARP is ahead of the stratified control at large budgets, but the evidence is weak. The SemanticScore difference at $n = 50$ is about $1.6$ standard errors, the other differences are smaller in standard-error units, and with three seeds, and with per-seed SHARP results from the original run not retained, no paired test is possible. At $n = 5$, where the two banks are identical, the arms agree to within $0.005$, so endpoint drift does not explain the gap. We therefore report the direction without claiming it.

\subsection{Case 2: Public-Benchmark Validation on BIRD}
\label{sec:bird}

This second case study repeats the evaluation on the BIRD text-to-SQL benchmark \citep{bird2024} with every proprietary component swapped for a public one, to check that SHARP's advantages are not an artifact of the enterprise dataset or of the embedding endpoint.

\paragraph{Setup.}
{\sloppy We use the BIRD training split (9,428 question-SQL pairs across 69
databases). The retrieval simulation and the end-to-end LLM evaluation both use the same seven databases,
\texttt{retail\_world}, \texttt{public\_review\_platform},
\texttt{movie\_3}, \texttt{soccer\_2016}, \texttt{synthea},
\texttt{mondial\_geo}, and \texttt{legislator}, covering retail, reviews,
film, sports, health, geography, and government domains respectively
(1,992 examples total). Embeddings are computed with
\texttt{sentence-transformers/\allowbreak all-mpnet-base-v2} (768-d; no proprietary
endpoint). SQL is generated by \texttt{meta/llama-3.1-70b-instruct} via a
public NVIDIA NIM endpoint. The cold-start seed domain is
\texttt{retail\_world}. All other hyperparameters ($n_0{=}5$, $B{=}5$,
total budget $n{=}50$, 30-dim reduction, Mat\'ern-$1/2$, median-heuristic length scale)
are identical to Case~1 (Section~\ref{sec:enterprise}).\par}

One setting differs from Case~1 and affects the reading of Remark~\ref{rem:label_dependence}.
No self-consistency noise prior was computed for BIRD. In both BIRD evaluations,
$\hat{\sigma}^2(x)$ is instead set to $0.03 + 0.5$ times the variance of the oracle
labels $\tilde{y}$ of the five labeled examples nearest to $x$. Selection on BIRD
therefore depends on labels through $\hat{\Lambda}$, and since $\tilde{y}$ is
computed against the held-out test SQL, the test set influences selection
indirectly. As a check, we reran the BIRD retrieval simulation with a constant,
label-free $\sigma^2(x) = 0.05$. The final Non-Seed Domain Score changes from
$0.545$ to $0.539$ (paired difference $0.006 \pm 0.004$ SE over 10 seeds) and its
AUC from $0.464$ to $0.463$, so the label-dependent noise estimate accounts for
little of the result.

\subsubsection{Retrieval Simulation}

Table~\ref{tab:bird_sim} reports the retrieval simulation over 10 seeds. The seven BIRD databases are semantically distinct enough that every method reaches 7/7 coverage in the first rounds, so the domain-discovery metric tells us nothing here. The table tracks the Non-Seed Domain Score across budgets, the same metric as Figure~\ref{fig:learning_curves_sim} for the enterprise pool, together with its area under the curve; Figure~\ref{fig:bird_sim} plots the learning curve.

\begin{table}[h]
\centering
\caption{BIRD retrieval simulation (mean $\pm$ 1 SE over 10 seeds). All methods achieve 7/7 domain coverage throughout. Columns $n{=}10$--$50$: Non-Seed Domain Score. AUC: area under the Non-Seed Domain Score curve divided by the total budget. Best per column in \textbf{bold}.}
\label{tab:bird_sim}
\resizebox{\linewidth}{!}{%
\begin{tabular}{lccccc}
\toprule
\textbf{Method} & \multicolumn{4}{c}{\textbf{Non-Seed Domain Score}} & \textbf{AUC} \\
 & $n{=}10$ & $n{=}20$ & $n{=}30$ & $n{=}50$ & \\
\midrule
Random            & $0.473_{\pm.005}$ & $0.491_{\pm.008}$ & $0.508_{\pm.009}$ & $0.531_{\pm.007}$ & 0.453 \\
Dist-to-Threshold & $0.476_{\pm.005}$ & $0.491_{\pm.005}$ & $0.506_{\pm.005}$ & $0.524_{\pm.006}$ & 0.450 \\
\midrule
Stratified Random & $0.482_{\pm.006}$ & $0.499_{\pm.005}$ & $0.518_{\pm.006}$ & $0.537_{\pm.004}$ & 0.459 \\
Stratified DtT    & $0.475_{\pm.005}$ & $0.496_{\pm.003}$ & $0.514_{\pm.005}$ & $0.528_{\pm.007}$ & 0.454 \\
\textbf{SHARP}    & $\mathbf{0.484_{\pm.004}}$ & $\mathbf{0.508_{\pm.005}}$ & $\mathbf{0.527_{\pm.006}}$ & $\mathbf{0.545_{\pm.007}}$ & \textbf{0.464} \\
\bottomrule
\end{tabular}}
\end{table}

\begin{figure}[h]
\centering
\includegraphics[width=0.60\textwidth]{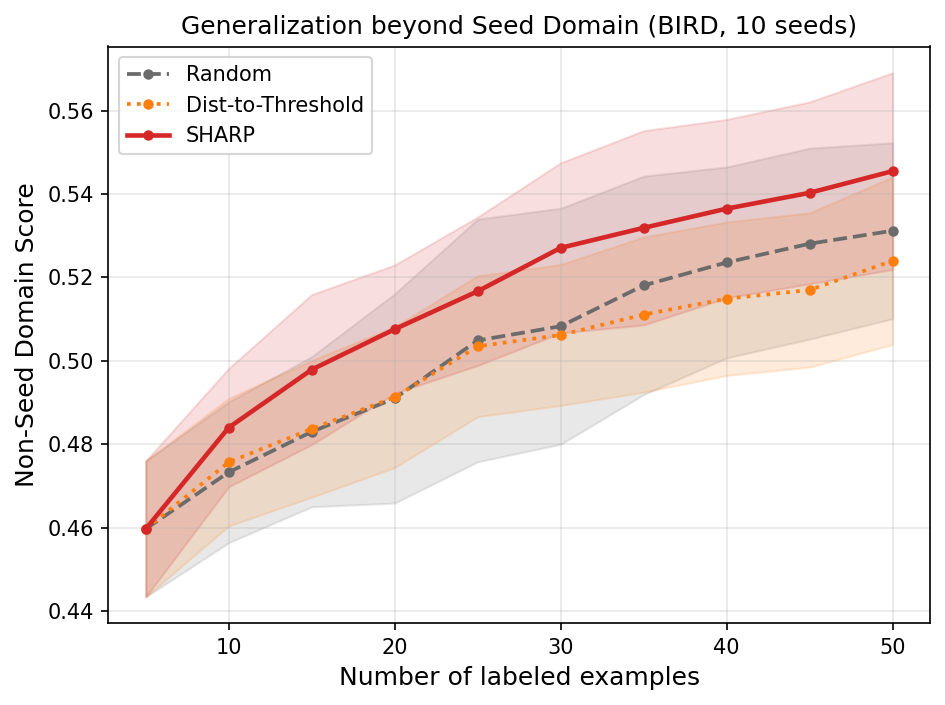}
\caption{BIRD retrieval simulation: Non-Seed Domain Score (mean $\pm1$ s.d., 10 seeds) for SHARP and the unstratified baselines; the stratified controls are reported in Table~\ref{tab:bird_sim}. SHARP leads from $n=10$ onward, finishing at $0.545$ vs.\ $0.531$ for Random and $0.524$ for Dist-to-Threshold.}
\label{fig:bird_sim}
\end{figure}

SHARP leads from $n=10$ onward on Non-Seed Domain Score, consistent with the
enterprise experiment. The advantage at $n=50$
($0.545$ vs.\ $0.531$ for Random and $0.524$ for Dist-to-Threshold) is
smaller than in the enterprise setting because the BIRD databases are far apart semantically and even random selection covers the non-seed domains reasonably well. The gap is consistent, though: SHARP has the
highest score at every budget from $n=10$ on and the highest AUC ($0.464$
vs.\ $0.453$ for Random), although at each budget the margin is only one to two
standard errors of the difference. Against the stratified
controls the margin is smaller still. Stratified Random finishes at $0.537$
(paired difference $0.008$, $p = 0.40$ over seeds), and Stratified
Dist-to-Threshold at $0.528$ ($0.017$, $p = 0.002$); on this pool the
farthest-point rule is the weakest within-stratum choice, consistent with the
flat locality profile reported in Section~\ref{sec:icl_bridge}. We do
not report a worst-case (minimum over domains) score on either pool. On the
enterprise pool one four-example domain is never annotated by any method
within the budget, so the minimum is pinned to that domain's floor and does
not distinguish methods; on BIRD the three methods are within one standard
error of each other on it.

\subsubsection{End-to-End LLM Evaluation}

Table~\ref{tab:bird_llm} reports SemanticScore at each labeled-set size (mean $\pm$ 1 bootstrap SE, 10 seeds, 580 total test queries per round). Table~\ref{tab:bird_breakdown} gives the structural breakdown at $n=50$.

\begin{table}[h]
\centering
\caption{BIRD end-to-end LLM evaluation: SemanticScore at each labeled-set size (mean $\pm$ 1 bootstrap SE, 10 seeds, 58 test queries each). Strat.\ Random shares SHARP's partition, retrieval depth, prompt, splits, and seed bank; at $n = 5$ the banks are identical, so the $0.004$ difference there is endpoint drift between the two runs. Best per row in \textbf{bold}.}
\label{tab:bird_llm}
\small
\begin{tabular}{lcccc}
\toprule
\textbf{$n$} & \textbf{Random} & \textbf{Dist-to-Threshold} & \textbf{Strat.\ Random} & \textbf{SHARP} \\
\midrule
5  & $0.696_{\pm.008}$ & $0.697_{\pm.008}$ & $\bm{0.698_{\pm.008}}$ & $0.694_{\pm.008}$ \\
10 & $0.702_{\pm.008}$ & $0.693_{\pm.008}$ & $\bm{0.705_{\pm.008}}$ & $0.698_{\pm.008}$ \\
15 & $0.698_{\pm.008}$ & $0.696_{\pm.008}$ & $\bm{0.707_{\pm.008}}$ & $0.698_{\pm.008}$ \\
20 & $0.699_{\pm.008}$ & $0.696_{\pm.008}$ & $\bm{0.712_{\pm.008}}$ & $0.701_{\pm.008}$ \\
25 & $0.701_{\pm.008}$ & $0.697_{\pm.008}$ & $\bm{0.711_{\pm.008}}$ & $0.707_{\pm.008}$ \\
30 & $0.700_{\pm.008}$ & $0.695_{\pm.008}$ & $\bm{0.716_{\pm.008}}$ & $0.710_{\pm.008}$ \\
35 & $0.699_{\pm.008}$ & $0.695_{\pm.008}$ & $\bm{0.718_{\pm.008}}$ & $0.713_{\pm.008}$ \\
40 & $0.706_{\pm.008}$ & $0.697_{\pm.008}$ & $\bm{0.721_{\pm.008}}$ & $0.714_{\pm.008}$ \\
45 & $0.705_{\pm.008}$ & $0.697_{\pm.008}$ & $\bm{0.719_{\pm.008}}$ & $0.715_{\pm.008}$ \\
50 & $0.701_{\pm.008}$ & $0.700_{\pm.008}$ & $\bm{0.721_{\pm.008}}$ & $0.720_{\pm.008}$ \\
\bottomrule
\end{tabular}
\end{table}

\begin{table}[h]
\centering
\caption{BIRD structural breakdown at $n=50$ (mean $\pm1$ bootstrap SE, 10 seeds, 580 total test queries). Best per row in \textbf{bold}.}
\label{tab:bird_breakdown}
\begin{tabular}{lcccc}
\toprule
\textbf{Metric} & \textbf{Random} & \textbf{Dist-to-Threshold} & \textbf{Strat.\ Random} & \textbf{SHARP} \\
\midrule
SemanticScore      & $0.701_{\pm0.008}$ & $0.700_{\pm0.008}$ & $\bm{0.721_{\pm0.008}}$ & $0.720_{\pm0.008}$ \\
Table Match Rate   & $0.595_{\pm0.021}$ & $0.597_{\pm0.021}$ & $0.659_{\pm0.019}$ & $\bm{0.666_{\pm0.020}}$ \\
Column Coverage    & $0.575_{\pm0.020}$ & $0.587_{\pm0.019}$ & $\bm{0.608_{\pm0.019}}$ & $\bm{0.608_{\pm0.020}}$ \\
Execution Accuracy & $0.328_{\pm0.020}$ & $0.326_{\pm0.019}$ & $0.333_{\pm0.020}$ & $\bm{0.345_{\pm0.020}}$ \\
\bottomrule
\end{tabular}
\end{table}

\begin{figure}[h]
\centering
\includegraphics[width=0.95\textwidth]{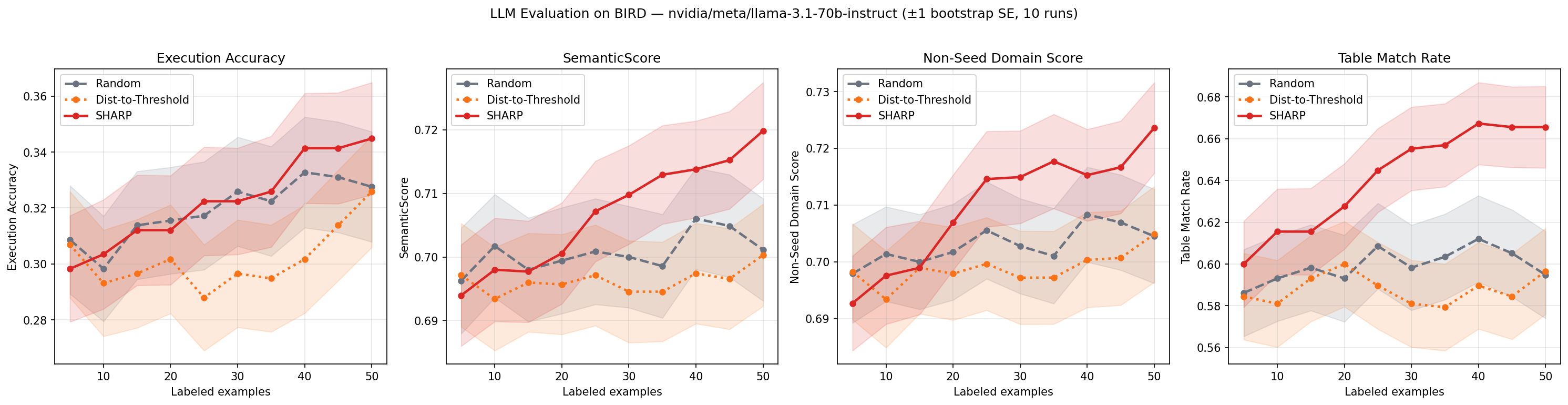}
\caption{BIRD LLM evaluation learning curves for SHARP and the unstratified baselines ($\pm1$ bootstrap SE, 10 seeds, 58 test queries each; Stratified Random is reported in Tables~\ref{tab:bird_llm} and~\ref{tab:bird_breakdown}). SHARP leads on SemanticScore from $n=20$ onward and on Table Match Rate at every budget. Execution accuracy differences are within one standard error across all methods throughout the budget.}
\label{fig:bird_llm}
\end{figure}

The three methods are indistinguishable for $n \le 15$, where the
labeled set is dominated by the shared seed examples. SHARP leads on
SemanticScore at every budget from $n=20$ onward and the gap widens with
budget, finishing at $0.720 \pm 0.008$ vs.\ $0.701 \pm 0.008$ for Random
and $0.700 \pm 0.008$ for Dist-to-Threshold (+1.9 pp over Random, +2.0 pp
over Dist-to-Threshold). The structural breakdown at $n=50$ shows SHARP
leading on Table Match Rate ($0.666$ vs.\ $0.595$ for Random, +7.1 pp)
and Column Coverage ($0.608$ vs.\ $0.575$, +3.3 pp over Random),
consistent with the enterprise experiment. Paired Wilcoxon tests over the 580 pooled test queries
at $n=50$ give $p = 0.003$ against Random and $p < 0.001$ against
Dist-to-Threshold for SemanticScore, and $p < 0.001$ against both for Table
Match Rate; the Column Coverage margin is significant against Random
($p = 0.04$) but not against Dist-to-Threshold ($p = 0.18$). Execution
accuracy differences are within one bootstrap standard error ($p \approx 0.25$
against both baselines), so the SHARP advantage there ($0.345$ vs.\ $0.328$
for Random) is not established at this test-set size. Non-Seed Domain Score follows SemanticScore closely,
reflecting that BIRD databases are semantically distinct enough that
random selection generalizes adequately. At $n=50$ SHARP has the highest
per-seed SemanticScore in 7 of the 10 seeds (among the three methods compared in that original run). On public data, then, a stratified bank yields better SQL structure than an unstratified one, mainly
through table selection, without a demonstrated gain in execution accuracy.

The stratified control places this gain in the partition. Stratified Random is at
least as good as SHARP at every budget in Table~\ref{tab:bird_llm}, and at $n = 50$
the two are indistinguishable on every metric: SemanticScore $0.721$ vs.\ $0.720$
($p = 0.90$), Table Match Rate $0.659$ vs.\ $0.666$ ($p = 0.69$), Column Coverage
$0.608$ for both ($p = 0.87$), and Execution Accuracy $0.333$ vs.\ $0.345$
($p = 0.43$). The SemanticScore AUC over the full budget favors Stratified Random
by $0.006$ ($p = 0.11$). On both pools, both stratified methods, SHARP and
Stratified Random, score above Random and Dist-to-Threshold at every budget
from $n = 30$ onward.

\subsection{Ablation and Stratified Controls}
\label{sec:ablation}

Algorithm~\ref{alg:greedy} combines three design choices: the information-gain objective $I(Y_S; f_C)$ of Section~\ref{sec:objective} (score $\delta_x = \mathrm{Var}(Y_x \mid Y_S)/\sigma^2(x)$), the partition matroid constraint, and the 30-dimensional UMAP projection on which the GP and the clusters are built. To find out where the gains come from, we rerun the retrieval simulation of Sections~\ref{sec:experiments} and~\ref{sec:bird} over the full $2 \times 2 \times 2$ factorial of these choices, holding everything else fixed (same seeds, splits, kernel, and noise prior). The alternative objective is the mutual information $I(Y_S; Y_{C \setminus S})$ with score $\mathrm{Var}(Y_x \mid Y_S)/\mathrm{Var}(Y_x \mid Y_{\bar S})$, which is not monotone and carries no approximation guarantee under the matroid. Removing the matroid collapses all pool points into a single block, so the algorithm reduces to unconstrained greedy selection on the same score. The alternative projection is a linear PCA to the same 30 dimensions. Baselines do not depend on any of the three choices and are unchanged from Tables~\ref{tab:sim_results} and~\ref{tab:bird_sim}.

\begin{table}[h]
\centering
\caption{Ablation of Algorithm~\ref{alg:greedy} on the retrieval simulation (10 seeds). Non-Seed Domain Score at $n=50$ (mean $\pm$ s.d.\ across seeds) and its AUC; CovAUC is the Coverage AUC of Table~\ref{tab:sim_results} (enterprise only; all BIRD arms cover 7/7 throughout). The first row is the configuration reported in the rest of the paper. The last two rows keep the partition and fill rule and replace the GP score inside each stratum by a uniform draw or by the max-min-distance rule. Baseline reference: enterprise Random $0.568$, Dist-to-Threshold $0.634$; BIRD Random $0.531$, Dist-to-Threshold $0.524$.}
\label{tab:ablation}
\small
\resizebox{\linewidth}{!}{%
\begin{tabular}{lllccccc}
\toprule
& & & \multicolumn{3}{c}{\textbf{Enterprise}} & \multicolumn{2}{c}{\textbf{BIRD}} \\
\cmidrule(lr){4-6}\cmidrule(lr){7-8}
\textbf{Projection} & \textbf{Objective} & \textbf{Matroid} & NS@50 & NS AUC & CovAUC & NS@50 & NS AUC \\
\midrule
UMAP & Info.\ gain & on  & $\bm{0.690_{\pm.016}}$ & $0.579$ & $5.34$ & $0.545_{\pm.024}$ & $0.464$ \\
UMAP & Info.\ gain & off & $0.535_{\pm.050}$ & $0.436$ & $4.60$ & $0.548_{\pm.016}$ & $0.465$ \\
UMAP & MI          & on  & $0.683_{\pm.021}$ & $0.582$ & $5.34$ & $0.555_{\pm.018}$ & $0.469$ \\
UMAP & MI          & off & $0.600_{\pm.026}$ & $0.491$ & $5.19$ & $\bm{0.556_{\pm.017}}$ & $0.470$ \\
PCA  & Info.\ gain & on  & $0.684_{\pm.019}$ & $0.580$ & $5.34$ & $0.545_{\pm.020}$ & $0.465$ \\
PCA  & Info.\ gain & off & $0.482_{\pm.014}$ & $0.405$ & $4.48$ & $0.536_{\pm.021}$ & $0.459$ \\
PCA  & MI          & on  & $0.688_{\pm.013}$ & $0.583$ & $5.34$ & $0.548_{\pm.021}$ & $0.463$ \\
PCA  & MI          & off & $0.595_{\pm.021}$ & $0.477$ & $5.16$ & $0.544_{\pm.020}$ & $0.462$ \\
\midrule
\multicolumn{2}{l}{No GP: Stratified Random} & on & $\bm{0.690_{\pm.013}}$ & $0.583$ & $5.34$ & $0.537_{\pm.014}$ & $0.459$ \\
\multicolumn{2}{l}{No GP: Stratified DtT}    & on & $0.687_{\pm.019}$ & $0.581$ & $5.34$ & $0.528_{\pm.021}$ & $0.454$ \\
\bottomrule
\end{tabular}}
\end{table}

On the enterprise pool, Table~\ref{tab:ablation} attributes the gain to the partition matroid. The four arms with the constraint in place are statistically indistinguishable ($0.683$ to $0.690$ at $n=50$, within one standard deviation of each other) and share the same Coverage AUC of $5.34$. Since the partition is built from the seven domain labels, the constraint forces one pick per domain in the first batch regardless of objective or projection, so the coverage path is identical across these arms and only the within-domain choice differs. Removing the constraint costs between $0.08$ and $0.20$ at $n=50$; every unconstrained arm falls below Dist-to-Threshold ($0.634$), and the unconstrained information-gain arms fall below Random ($0.568$). This is a consequence of the pool composition. The seed domain holds 350 of the 477 examples, so an unconstrained greedy rule that favors low-noise, far-from-labeled points spends most of its budget inside that domain. Without the constraint the MI score degrades less than the information-gain score ($0.60$ against $0.48$ to $0.54$), because its denominator $\mathrm{Var}(Y_x \mid Y_{\bar S})$ penalizes points that are well explained by the rest of the pool and thereby spreads the selections; once the constraint is in place this difference disappears.

On this pool, then, the choice of objective does not affect performance, and what it contributes is the guarantee. Both objectives select essentially the same sets under the matroid, but only $I(Y_S; f_C)$ is monotone and submodular (Lemma~\ref{lemma:submodularity}), so Theorem~\ref{thm:matroid_ideal} applies to the reported greedy selections, in rounds without the fallback, only under that objective. The projection has no measurable effect on selection quality on either pool, with UMAP and PCA agreeing to within noise in every arm. Its function is to make the GP tractable in 30 dimensions rather than $d = 2048$; it does not determine which examples are chosen.

On BIRD all eight arms lie within $0.02$ of one another and within about one standard deviation of the baselines. This agrees with Section~\ref{sec:bird}: the seven databases are semantically distinct enough that every method, including Random, covers them from the first rounds, leaving neither the constraint nor the objective any room to act. The ablation therefore places SHARP's advantage where the earlier sections suggested, in pools with a dominant seed domain and under-represented strata, where the partition constraint forces a coverage that unconstrained scores do not deliver.

\subsubsection{Stratified Controls and Pre-Specified Tests}
\label{sec:strat_controls}

The factorial above varies the objective and the projection but always keeps a GP score inside the strata, so it cannot say whether that score adds anything once the partition is in place. The stratified controls in the last two rows of Table~\ref{tab:ablation} answer that question, and the answer is no. In the simulation, Stratified Random ties SHARP on the enterprise pool ($0.690$ vs.\ $0.690$) and trails it within noise on BIRD ($0.537$ vs.\ $0.545$, $p = 0.40$). In the end-to-end LLM evaluations (Tables~\ref{tab:llm_results} and~\ref{tab:bird_llm}), SemanticScore at $n = 50$ is $0.545$ vs.\ $0.543$ on the enterprise pool ($p = 0.99$) and $0.721$ vs.\ $0.720$ on BIRD ($p = 0.90$).

A tie at the final budget could hide an advantage in regimes that the main experiments do not probe. Before running any further experiment we wrote down three such regimes, the tests, and a decision rule. The regimes were (i) the seed domain itself, which holds 335 enterprise and 358 BIRD pool examples and inside which the partition gives no guidance; (ii) small budgets, $n \le 20$; and (iii) batches larger than the number of strata. The rule was that a GP advantage would be claimed only if a paired test gave $p < 0.05$ after Holm correction across the primary tests of each experiment, with the same sign on both pools or an explicitly stated pool-specific scope. All outcomes are reported, including ties and losses. Table~\ref{tab:prereg} summarizes them.

\begin{table}[h]
\centering
\caption{Pre-specified comparisons of SHARP against Stratified Random. Differences are SHARP minus Stratified Random, averaged over seeds; $p$ is from a paired $t$-test over the 10 seeds unless marked $^{w}$ (paired Wilcoxon over pooled test queries). ``Primary'' tests enter the Holm correction of their experiment. Seed-dom.\ is the mean best top-5 SemanticScore on test queries of the seed domain; NS is the Non-Seed Domain Score; AUC$_{\le 20}$ is the area under the SemanticScore curve over $n \in \{5, \dots, 20\}$, divided by 20.}
\label{tab:prereg}
\small
\resizebox{\linewidth}{!}{%
\begin{tabular}{llllrrl}
\toprule
\textbf{Experiment} & \textbf{Pool} & \textbf{Setting} & \textbf{Metric} & \textbf{Diff.} & \textbf{$p$} & \textbf{Role} \\
\midrule
Seed domain (sim.)       & Enterprise & $B = 5$, $n = 50$   & Seed-dom.\ final / AUC & $-0.027$ / $-0.012$ & $0.17$ / $0.12$ & Primary \\
                         & BIRD       & $B = 5$, $n = 50$   & Seed-dom.\ final / AUC & $-0.005$ / $-0.006$ & $0.76$ / $0.45$ & Primary \\
Small budget (sim.)      & Enterprise & $B = 5$             & NS at $n = 10$ / $20$  & $-0.020$ / $-0.015$ & $0.04$ / $0.07$ & Secondary \\
                         & BIRD       & $B = 5$             & NS at $n = 10$ / $20$  & $+0.002$ / $+0.009$ & $0.78$ / $0.22$ & Secondary \\
Large batch (sim.)       & Enterprise & $B = 10$, $n = 45$  & NS final / AUC         & $+0.005$ / $+0.008$ & $0.43$ / $0.17$ & Primary \\
                         &            &                     & Seed-dom.\ final / AUC & $-0.037$ / $-0.030$ & $0.10$ / $0.015$ & Primary \\
                         & BIRD       & $B = 10$, $n = 45$  & NS final / AUC         & $+0.001$ / $+0.003$ & $0.74$ / $0.51$ & Primary \\
                         &            &                     & Seed-dom.\ final / AUC & $-0.018$ / $-0.008$ & $0.024$ / $0.21$ & Primary \\
\midrule
Fill-slot, fresh seeds (LLM) & Enterprise & 2 strata, $B = 5$ & AUC$_{\le 20}$       & $+0.019$ & $0.041$ (Holm $0.081$) & Primary \\
                         &            &                     & SemanticScore at $n = 15$ & $+0.038$ & $0.008^{w}$ & Secondary \\
                         & BIRD       & 2 strata, $B = 5$   & AUC$_{\le 20}$         & $+0.001$ & $0.75$ & Primary \\
Fill-slot (sim.)         & Enterprise & 2 strata, $B = 5$   & NS AUC$_{\le 20}$      & $-0.034$ & $< 0.001$ & Sanity check \\
                         & BIRD       & 2 strata, $B = 5$   & NS AUC$_{\le 20}$      & $0.000$ & $0.92$ & Sanity check \\
\bottomrule
\end{tabular}}
\end{table}

None of the primary tests of the first three experiments survives the Holm correction (the smallest adjusted $p$ is $0.12$), and the two with nominal $p < 0.05$ both favor Stratified Random. Inside the seed domain, where one might expect an informed within-stratum choice to matter most, the GP score is never better than a uniform draw.

The fourth experiment was added afterwards, to test the one positive signal in the data: SHARP's lead over Stratified Random at $n = 10$ and $15$ in the enterprise LLM evaluation (Section~\ref{sec:llm_eval}). Because that signal was found by looking, we stated a hypothesis for it and tested it on fresh data. In the enterprise LLM evaluation only two domains remain, so each batch of five has three slots that the partition leaves unconstrained and that SHARP fills by its GP score. The hypothesis was that the GP helps precisely when the batch exceeds the number of strata. We tested it with new seeds (52 to 61, giving new splits and seed banks), all arms run on the same day, on the enterprise pool and on BIRD with its seven databases coarsened into two strata (the seed database and the rest). The hypothesis was not confirmed. On the enterprise pool the small-budget difference points the same way as before ($+0.019$ in AUC$_{\le 20}$), but $p = 0.041$ becomes $0.081$ after Holm correction; on BIRD the difference is $+0.001$ ($p = 0.75$); and the simulation run as a sanity check points the other way on the enterprise pool ($-0.034$). The most telling result is that the stratified farthest-point rule matches SHARP on the enterprise pool ($+0.003$ in AUC$_{\le 20}$, $p = 0.63$) and itself beats Stratified Random ($+0.016$, $p = 0.011$), whereas on BIRD the three stratified rules are tied (SHARP minus stratified farthest-point $0.000$, $p = 0.98$; farthest-point minus Stratified Random $+0.001$, $p = 0.83$).
Whatever small-budget advantage exists on the enterprise pool is therefore a property of spreading the fill slots out geometrically, which the farthest-point rule does without a GP.

So the controls tell us what the gains in this paper are made of. Domain coverage, imposed by the partition, accounts for all of the improvement over unstratified selection on both pools in the retrieval simulation and the end-to-end LLM evaluation. The LLM-scored-label experiment of Section~\ref{sec:realistic} is the one exception in direction, with SHARP ahead of Stratified Random at larger budgets, but it rests on three seeds and was not part of the pre-specified tests. Spreading picks by distance within the unconstrained slots helps on the enterprise pool at small budgets and not on BIRD. The noise weighting that distinguishes SHARP's score from the farthest-point rule has no measurable effect on these pools.

\subsection{Discussion}

The results agree across both case studies and all three enterprise evaluations, and thanks to the stratified controls of Section~\ref{sec:strat_controls} we can say which part of Algorithm~\ref{alg:greedy} is responsible for them.

Coverage is what drives the gains. Against unstratified selection, Algorithm~\ref{alg:greedy} reaches 6/7 domains at $n = 10$ rather than $n = 15$ in the enterprise simulation, lifts the end-to-end Non-Seed Domain Score to $0.608$ against $0.528$ for Dist-to-Threshold, and on BIRD lifts SemanticScore to $0.720$ against $0.700$, mostly through table selection. Stratified Random, which uses the same partition and picks uniformly inside each stratum, matches each of these gains: end to end it scores $0.545$ vs.\ $0.543$ in SemanticScore and $0.593$ vs.\ $0.608$ in Non-Seed Domain Score on the enterprise pool ($p = 0.99$ and $0.51$), and $0.721$ vs.\ $0.720$ on BIRD. The ablation says the same thing from the other direction, since removing the constraint removes the advantage whatever the objective or projection. What a practitioner should take from this paper is the partition: one annotated example per domain per batch, from the first batch on.

As for the GP score, inside the strata it was never better than a uniform draw in any pre-specified comparison. The one positive signal, a lead at $n \le 15$ in the enterprise LLM evaluation, did not survive correction for multiple testing on fresh seeds, and where it does appear a stratified farthest-point rule with no GP matches it. Lemma~\ref{lemma:one_neighbor} lets us read SHARP's score as a noise-weighted farthest-point rule, so this result isolates the noise weighting as the component with no measurable effect on these pools. The GP still does work in the paper, but work of a different kind. It is the model from which the selection rule is derived (Lemma~\ref{lemma:marginal_info_gain}), from which the guarantees of Theorems~\ref{thm:matroid_ideal} and~\ref{thm:robust_matroid} follow, and within which the retrieval model of Section~\ref{sec:icl_bridge} picks out coverage, fill distance, and noise as the terms that govern downstream accuracy. Those terms predicted, after the fact, the pattern the controls show. For the GP to improve selection rather than explain it, it would have to act on something geometry does not carry, a posterior mean in a label-adaptive rule or annotation noise that genuinely varies across the pool, and neither was present here.

On the projection: with $\hat{d}_I = 18.4 \ll d = 2048$, 30 dimensions are enough to keep the intrinsic geometry the GP kernel depends on, and the ablation (Section~\ref{sec:ablation}) shows that UMAP and PCA give selection quality indistinguishable within noise once the matroid is in place. The projection makes the GP tractable; it does not decide the selection. Since the partition was given by domain labels in every reported run, the manifold-aware spectral clustering of Section~\ref{sec:geometric_design} was never exercised.

On misspecification (Theorem~\ref{thm:robust_matroid}): the Mat\'ern-$1/2$ surrogate was chosen so that Assumption~\ref{assum:spectral} holds, with $s = \nu^* - 1/2 < \nu^*$, against any Mat\'ern-$\nu^*$ truth with $\nu^* \ge 1/2$. A squared-exponential truth, or a surrogate smoother than the truth, falls outside the assumption. Within its scope the theorem gives the form of the degradation, an additive per-step discrepancy $\eta = \frac{1}{2}\log(c_2/c_1) + \varepsilon_s$. But the diagnostics of Appendix~\ref{app:thm3_diag} show that on both pools the constants are far too large for the bound to certify anything: $B\eta$ exceeds $30$ nats against a per-batch objective below $9$, and for $s > 0$ the hypothesis of Lemma~\ref{lemma:variance_mismatch}(iii) fails outright. The robustness we see is an empirical fact, and the stratified controls explain it, since the gains come from a partition built from domain labels that never involves the kernel.

Finally, LLM-scored labels versus oracle labels. The LLM-scored-label experiment is the strictest test of the pipeline: no ground-truth SQL is seen anywhere in the loop and every label carries LLM measurement noise. Algorithm~\ref{alg:greedy} leads both unstratified baselines on SemanticScore from $n = 10$ onward and ends with a Table Match Rate of $0.533$ against $0.433$ for both. Stratified Random also beats both unstratified baselines here and matches SHARP's Table Match Rate at $n = 50$. SHARP's lead over it at larger budgets ($0.603$ vs.\ $0.569$ in SemanticScore) is the one place in the paper where the GP score may be adding something, but it rests on three seeds and is not statistically established. Because selection is label-free in the reported configuration (Remark~\ref{rem:label_dependence}), this result shows that the selected banks hold up under a different labeling and evaluation regime; it does not show that noisy labels improve acquisition. Whether refitting $\hat{\theta}$ or $\hat{\Lambda}$ from accumulating labels would give the GP a role that geometry cannot play is the open question we leave.

\section{Conclusion}
We have described a sequential experimental design framework, built on a Gaussian process surrogate, for choosing which few-shot examples to annotate in a text-to-SQL system. The curse of dimensionality that usually troubles spatial statistics in high dimension is mitigated for the GP fit because LLM embeddings live on a low-dimensional intrinsic manifold, although it returns in the misspecification constants, which grow exponentially with the dimension. Example selection becomes monotone submodular information-gain maximization under a partition matroid, with a $\tfrac{1}{2}$-approximation guarantee for stratified greedy selection, and we derived an additive degradation bound under kernel misspecification. The bound identifies the spectral ratio and the roughening residual as what governs the degradation; on our pools its constants are too large for it to be quantitatively useful.

On a production supply-chain dataset and on BIRD, the stratified selection improves over random and distance-based selection at the same annotation budget: it reaches 6/7 domain coverage at $n = 10$ rather than $n = 15$ in the enterprise simulation, raises the end-to-end Non-Seed Domain Score from $0.528$ to $0.608$ on the enterprise pool, and raises SemanticScore from $0.700$ to $0.720$ on BIRD. Stratified controls and pre-specified tests locate all of this improvement in the partition. A uniform draw inside each domain matches Algorithm~\ref{alg:greedy} on both pools in the retrieval simulation and the end-to-end evaluation, and trails it only in the three-seed LLM-scored-label experiment, by a margin that is not statistically established; a farthest-point rule inside each domain adds a small-budget gain where locality is steep; and the noise weighting of the GP score adds nothing measurable in any pre-specified test. The retrieval model of Section~\ref{sec:icl_bridge} accounts for this pattern. Its lower bound is governed by domain coverage, by within-domain fill distance weighted by the calibrated locality constant, and by demonstration noise, whose calibrated support is weak.

The recommendation to practitioners is short. Annotate one example per domain per batch from the first batch on, and when embedding locality is steep, spread the picks by distance inside each domain. The GP framework provides a principled derivation of the within-stratum rule, its guarantees, and an account of when each ingredient matters; empirically, however, the dominant benefit comes from the partition constraint itself, and on these pools the GP score does not improve on the rule's geometric core. Two extensions might change that. One is label-adaptive acquisition, in which the posterior mean or refitted noise estimates steer selection toward regions where the LLM fails. The other is pools in which annotation noise varies enough for the noise weighting to matter. Separately, the partition could be updated as labeled data accumulates, so that it adapts to discovered semantic structure rather than relying on the initial pool embedding.
\appendix

\section{Mathematical Background}
\label{app:background}
This appendix collects the spectral and combinatorial facts used in the derivations. For a detailed treatment, we refer the reader to \citet{stein1999interpolation} for kernel theory and \citet{krause2014submodular} for submodularity.

\subsection{Sobolev Spaces and Spectral Kernel Representation}
A stationary covariance kernel $K(x, x') = \Phi(h)$ with $h = x - x'$ is uniquely characterized by its spectral density $\hat{\Phi}(\omega)$ via Bochner's Theorem. We define the Fourier transform of the kernel as:
\begin{equation}
    \hat{\Phi}(\omega) = \frac{1}{(2\pi)^d} \int_{\mathbb{R}^d} \Phi(h) e^{-i \omega^T h} dh
\end{equation}
The spatial kernel is dually recovered via the inverse Fourier integral, $\Phi(h) = \int_{\mathbb{R}^d} \hat{\Phi}(\omega) e^{i \omega^T h} d\omega$. The smoothness of the functions in the induced Reproducing Kernel Hilbert Space (RKHS) is determined by the asymptotic decay of $\hat{\Phi}(\omega)$. Specifically, a kernel is of order $\tau$ if its spectral density satisfies $\hat{\Phi}(\omega) \asymp (1 + \|\omega\|^2)^{-\tau}$, inducing an RKHS norm-equivalent to the fractional Sobolev space $H^\tau(\mathbb{R}^d)$.

A primary example is the Matérn class of kernels, defined in the spectral domain as:
\begin{equation}
    \hat{\Phi}_{\text{Matérn}}(\omega) = \frac{\phi\, \Gamma(\nu + d/2)}{\pi^{d/2}\, \Gamma(\nu)}\, \alpha^{2\nu} (\alpha^2 + \|\omega\|^2)^{-(\nu + d/2)}, \qquad \alpha = \frac{\sqrt{2\nu}}{\ell}, \label{eq:matern_spectral}
\end{equation}
where $\nu$ is the smoothness parameter, $\phi = \Phi(0)$ is the variance, $\ell$ is the length scale, and the constant makes $\int \hat{\Phi} = \phi$ under the transform convention above. This is the standard parameterization, in which $\Phi$ depends on $h$ through $\sqrt{2\nu}\,\|h\|/\ell$; for $\nu = 1/2$ it gives the exponential kernel $\phi \exp(-\|h\|/\ell)$ used in Section~\ref{sec:experiments}. Two Mat\'ern kernels with the same $\ell$ but different $\nu$ therefore have different $\alpha$, and the diagnostics of Appendix~\ref{app:thm3_diag} use this convention throughout. Under this formulation, the Mat\'ern kernel corresponds to a Sobolev space of order $\tau = \nu + d/2$. This spectral decay governs the interpolation error. When restricted to a $d_I$-dimensional manifold $\mathcal{M}$, the approximation error scales with the geodesic fill distance $h_{\mathcal{S}, \mathcal{M}}$ at a rate determined by the Sobolev index $\tau$ and the intrinsic dimension $d_I$ \citep{fuselier2012scattered}. Multiplying a spectral density by $(1 + \|\omega\|^2)^s$, as in the roughened kernel $\mathcal{K}^*_s$ of Section~\ref{sec:misspecified_theory}, lowers the Sobolev order from $\tau$ to $\tau - s$ and yields a valid kernel precisely when $\tau - s > d/2$, i.e.\ when the spectral moment of order $s$ is finite.

\subsection{Matroids and Constrained Set Selection}
\label{app:matroids}
A matroid abstracts the notion of linear independence, and it is the setting in which greedy maximization of a submodular function has constant-factor guarantees. Formally, a matroid is a pair $M = (E, \mathcal{I})$ with finite ground set $E$ and family $\mathcal{I} \subseteq 2^E$ of \emph{independent sets} satisfying (i) the \emph{hereditary property}: $A \in \mathcal{I}$ and $B \subseteq A$ imply $B \in \mathcal{I}$; and (ii) the \emph{exchange property}: if $A, B \in \mathcal{I}$ with $|A| < |B|$, there exists $x \in B \setminus A$ such that $A \cup \{x\} \in \mathcal{I}$.

Two instances arise in our setting. The \emph{uniform matroid} $U_{n,k}$ has independent sets $\mathcal{I} = \{S \subseteq E : |S| \le k\}$, encoding a pure cardinality constraint. The \emph{partition matroid} is defined by a disjoint partition $E = \bigsqcup_{i=1}^K C_i$ with integer capacities $d_i \ge 0$, and has independent sets
\begin{equation}
    \mathcal{I} = \left\{ S \subseteq E \,:\, |S \cap C_i| \le d_i \,\,\, \forall i = 1, \dots, K \right\}.
\end{equation}
The partition matroid strictly generalizes the uniform matroid ($K = 1$, $d_1 = k$ recovers $U_{n,k}$). In our framework, the blocks $C_1, \dots, C_K$ correspond to semantic clusters on the embedding manifold, and the per-block capacity $d_i = 1$ enforces selection of at most one example per stratum.

\subsection{Entropy and Submodular Optimization}
For a continuous random variable $Y$ with probability density function $p(y)$, the differential entropy $H(Y)$ is defined as:
\begin{equation}
    H(Y) = -\int_{\mathcal{Y}} p(y) \log p(y) dy
\end{equation}
For a multivariate Gaussian random variable $Y_{\mathcal{S}} \sim \mathcal{N}(\mu, \Sigma_{\mathcal{S}})$ indexed by a set $\mathcal{S}$, this integral evaluates to:
\begin{equation}
    H(Y_{\mathcal{S}}) = \frac{1}{2}\log\left((2\pi e)^{|\mathcal{S}|} \det(\Sigma_{\mathcal{S}})\right)
\end{equation}
Mutual information between two random variables $Y_A$ and $Y_B$ is defined as the Kullback--Leibler (KL) divergence between their joint distribution and the product of their marginals:
\begin{equation}
    I(Y_A; Y_B) = D_{\mathrm{KL}}\!\left(p(Y_A, Y_B) \;\big\|\; p(Y_A)\,p(Y_B)\right)
\end{equation}
This quantity measures how much the joint distribution departs from independence. Expanding the KL divergence using the definition of conditional entropy yields the equivalent form $I(Y_A; Y_B) = H(Y_A) - H(Y_A \mid Y_B)$.

Applying this definition to our setting, the objective is $F(\mathcal{S}) = I(Y_{\mathcal{S}}; f_{\mathcal{C}})$, the information that the noisy annotations at $\mathcal{S}$ carry about the latent correctness values on the pool. Because $Y_{\mathcal{S}} = f_{\mathcal{S}} + \epsilon_{\mathcal{S}}$ with $\epsilon$ independent of $f$, this equals $H(Y_{\mathcal{S}}) - H(\epsilon_{\mathcal{S}}) = \frac{1}{2}\log\det(I + \Lambda_{\mathcal{S}}^{-1}\Sigma_{\mathcal{S}\mathcal{S}})$. It is a normalized, monotone, submodular set function (Lemma~\ref{lemma:submodularity}), satisfying the property of diminishing returns: for all $A \subseteq B \subseteq V$ and $x \notin B$,
\begin{equation}
    F(A \cup \{x\}) - F(A) \ge F(B \cup \{x\}) - F(B)
\end{equation}
Monotone submodularity is exactly the hypothesis under which greedy selection subject to a matroid constraint $\mathcal{I}$, here the partition matroid restricting selection counts within disjoint query clusters, retains a constant-factor $(1/2)$-approximation guarantee relative to the global optimum \citep{fisher1978analysis}.

\section{Proofs of Main Results}
\label{app:proofs}
\begin{proof}[Proof of Theorem \ref{thm:np_hard_hetero_mi}]
\emph{Membership in NP.} Given a certificate $\mathcal{S}$, verifying the partition constraint takes linear time, and since the inputs and the threshold $\tau$ are rational, $\det(I + \Lambda_{\mathcal{S}}^{-1}\Sigma_{\mathcal{S}\mathcal{S}})$ is computed exactly in polynomial time by Gaussian elimination over $\mathbb{Q}$ and compared with $\tau$.

\emph{Hardness.} We reduce from \textsc{Multicolored Clique}: given a graph $G = (V, E)$ whose vertex set is partitioned into $K$ color classes $V_1, \dots, V_K$, decide whether $G$ contains a clique with exactly one vertex of each color. This problem is NP-complete. An instance $(G', k)$ of \textsc{Clique} maps to it by taking $K = k$ copies of $V(G')$ as the color classes and joining $u^{(i)}$ to $v^{(j)}$ whenever $i \neq j$, $u \neq v$, and $\{u, v\} \in E(G')$; a multicolored clique in the new graph is a set of $k$ distinct, pairwise adjacent vertices of $G'$, and conversely.

Given $(G, V_1, \dots, V_K)$ with $n = |V|$, set $\mathcal{C} = V$, $C_i = V_i$, and $\sigma^2(x) = 1$ for all $x$, so $\Lambda = I$. The theorem places no budget on $|\mathcal{S}|$, so throughout the proof $B = K$ and $\mathcal{I}$ is the untruncated partition matroid. Without loss of generality every color class is non-empty (otherwise the answer is trivially no). Let $N$ be the symmetric $0/1$ matrix indexed by $V$ with $N_{uu} = 0$ and, for $u \neq v$, $N_{uv} = 1$ if and only if $u$ and $v$ lie in different color classes and $\{u, v\} \notin E$. Define
\begin{equation}
    \Sigma = I + \rho N, \qquad \rho = \frac{1}{4n}.
\end{equation}
Since $\|N\|_2 \le n - 1$, $\|\rho N\|_2 < 1/4$ and $\Sigma$ is positive definite, hence a valid kernel matrix on the finite pool. Set the threshold $\tau = 2^K$ and write $M = \frac{1}{2}\log\tau = \frac{K}{2}\log 2$. The construction is polynomial in the size of $G$.

Consider $\mathcal{S} \in \mathcal{I}$ with $|\mathcal{S}| = K$, one vertex per color class. Then
\begin{equation}
    F(\mathcal{S}) = \tfrac{1}{2}\log\det(I + \Sigma_{\mathcal{S}\mathcal{S}}) = \tfrac{1}{2}\log\det(2I + \rho N_{\mathcal{S}}) = \tfrac{K}{2}\log 2 + \tfrac{1}{2}\log\det(I + \rho' N_{\mathcal{S}}), \qquad \rho' = \rho/2,
\end{equation}
where $N_{\mathcal{S}}$ is the principal submatrix of $N$ on $\mathcal{S}$. If $\mathcal{S}$ is a multicolored clique, every pair in $\mathcal{S}$ is adjacent, so $N_{\mathcal{S}} = 0$ and $F(\mathcal{S}) = M$. Otherwise $N_{\mathcal{S}} \neq 0$. Let $\lambda_1, \dots, \lambda_K$ be its eigenvalues. Then $\sum_i \lambda_i = \operatorname{tr} N_{\mathcal{S}} = 0$, $\sum_i \lambda_i^2 = \|N_{\mathcal{S}}\|_F^2 \ge 2$, and $|\rho' \lambda_i| \le \rho'(K-1) \le 1/8$. For $|z| \le 1/8$ we have $\log(1+z) \le z - z^2/2 + |z|^3/3 \le z - z^2/4$, so
\begin{equation}
    \log\det(I + \rho' N_{\mathcal{S}}) = \sum_{i=1}^K \log(1 + \rho'\lambda_i) \le \rho' \sum_i \lambda_i - \frac{\rho'^2}{4} \sum_i \lambda_i^2 \le -\frac{\rho'^2}{2} < 0,
\end{equation}
and therefore $F(\mathcal{S}) < M$.

Now let $\mathcal{S} \in \mathcal{I}$ be arbitrary with $F(\mathcal{S}) \ge M$. Because every color class is non-empty, $\mathcal{S}$ extends to a basis $\mathcal{S}' \supseteq \mathcal{S}$ with $|\mathcal{S}'| = K$, and by monotonicity (Lemma~\ref{lemma:submodularity}) $F(\mathcal{S}') \ge F(\mathcal{S}) \ge M$, which by the case analysis above forces $\mathcal{S}'$ to be a multicolored clique. Conversely, a multicolored clique $\mathcal{S}$ lies in $\mathcal{I}$ and attains $F(\mathcal{S}) = M$. Hence $G$ has a multicolored clique if and only if some $\mathcal{S} \in \mathcal{I}$ satisfies $F(\mathcal{S}) \ge M$, completing the reduction. Monotonicity of $F$ is what allows the decision problem over all of $\mathcal{I}$ (at most one query per cluster) to be settled by examining bases alone.
\end{proof}

\begin{proof}[Proof of Lemma \ref{lemma:marginal_info_gain}]
Write $Y_{\mathcal{S}} = f_{\mathcal{S}} + \epsilon_{\mathcal{S}}$ with $\epsilon_{\mathcal{S}} \sim \mathcal{N}(0, \Lambda_{\mathcal{S}})$ independent of $f$. Given $f_{\mathcal{S}}$, the vector $Y_{\mathcal{S}}$ is a function of $(f_{\mathcal{S}}, \epsilon_{\mathcal{S}})$ alone and hence conditionally independent of $f_{\mathcal{C} \setminus \mathcal{S}}$, so by the chain rule for mutual information
\begin{equation}
    F(\mathcal{S}) = I(Y_{\mathcal{S}}; f_{\mathcal{C}}) = I(Y_{\mathcal{S}}; f_{\mathcal{S}}) = H(Y_{\mathcal{S}}) - H(Y_{\mathcal{S}} \mid f_{\mathcal{S}}) = H(Y_{\mathcal{S}}) - H(\epsilon_{\mathcal{S}}). \label{eq:F_entropy_form}
\end{equation}
Under the Gaussian process, $Y_{\mathcal{S}} \sim \mathcal{N}(\mu_{\mathcal{S}}, \Sigma_{\mathcal{S}\mathcal{S}} + \Lambda_{\mathcal{S}})$ and $\epsilon_{\mathcal{S}} \sim \mathcal{N}(0, \Lambda_{\mathcal{S}})$, so
\begin{equation}
    F(\mathcal{S}) = \tfrac{1}{2}\log\det(\Sigma_{\mathcal{S}\mathcal{S}} + \Lambda_{\mathcal{S}}) - \tfrac{1}{2}\log\det\Lambda_{\mathcal{S}} = \tfrac{1}{2}\log\det(I + \Lambda_{\mathcal{S}}^{-1}\Sigma_{\mathcal{S}\mathcal{S}}),
\end{equation}
which is the closed form claimed, and gives $F(\emptyset) = 0$.

For the marginal gain, apply the entropy chain rule to both terms of Equation~\eqref{eq:F_entropy_form}: $H(Y_{\mathcal{S} \cup \{x\}}) = H(Y_{\mathcal{S}}) + H(Y_x \mid Y_{\mathcal{S}})$ and, by independence of the noise, $H(\epsilon_{\mathcal{S} \cup \{x\}}) = H(\epsilon_{\mathcal{S}}) + H(\epsilon_x)$. Hence
\begin{equation}
    \Delta F(x \mid \mathcal{S}) = H(Y_x \mid Y_{\mathcal{S}}) - H(\epsilon_x) = \tfrac{1}{2}\log\big(2\pi e \,\mathrm{Var}(Y_x \mid Y_{\mathcal{S}})\big) - \tfrac{1}{2}\log\big(2\pi e \,\sigma^2(x)\big) = \tfrac{1}{2}\log\frac{\mathrm{Var}(Y_x \mid Y_{\mathcal{S}})}{\sigma^2(x)}. \label{eq:mi_var_ratio}
\end{equation}
The predictive variance is the Schur complement of the block covariance of $(Y_x, Y_{\mathcal{S}})$:
\begin{equation}
    \mathrm{Var}(Y_x \mid Y_{\mathcal{S}}) = \Sigma_{xx} + \sigma^2(x) - \Sigma_{x\mathcal{S}} ( \Sigma_{\mathcal{S}\mathcal{S}} + \Lambda_{\mathcal{S}} )^{-1} \Sigma_{\mathcal{S}x}, \label{eq:var_S_schur}
\end{equation}
which gives the stated $\delta_x$. Finally, $\epsilon_x$ is independent of $(f, \epsilon_{\mathcal{S}})$ and therefore of $(f(x), Y_{\mathcal{S}})$, so $\mathrm{Var}(Y_x \mid Y_{\mathcal{S}}) = \mathrm{Var}(f(x) \mid Y_{\mathcal{S}}) + \sigma^2(x) \ge \sigma^2(x)$, i.e.\ $\delta_x \ge 1$.
\end{proof}

\begin{proof}[Proof of Lemma \ref{lemma:submodularity}]
\emph{Normalization} is $F(\emptyset) = 0$ from Lemma~\ref{lemma:marginal_info_gain}. \emph{Monotonicity}: by Lemma~\ref{lemma:marginal_info_gain}, $\Delta F(x \mid \mathcal{S}) = \frac{1}{2}\log\delta_x \ge 0$ for every $\mathcal{S}$ and $x \notin \mathcal{S}$.

\emph{Submodularity.} Fix $\mathcal{S} \subseteq \mathcal{S}' \subseteq \mathcal{C}$ and $x \in \mathcal{C} \setminus \mathcal{S}'$. By Equation~\eqref{eq:mi_var_ratio},
\begin{equation}
    \Delta F(x \mid \mathcal{S}) - \Delta F(x \mid \mathcal{S}') = \tfrac{1}{2}\log\frac{\mathrm{Var}(Y_x \mid Y_{\mathcal{S}})}{\mathrm{Var}(Y_x \mid Y_{\mathcal{S}'})},
\end{equation}
so it suffices to show $\mathrm{Var}(Y_x \mid Y_{\mathcal{S}'}) \le \mathrm{Var}(Y_x \mid Y_{\mathcal{S}})$. Let $\mathcal{T} = \mathcal{S}' \setminus \mathcal{S}$. For jointly Gaussian variables, conditioning on the additional block $Y_{\mathcal{T}}$ gives
\begin{equation}
    \mathrm{Var}(Y_x \mid Y_{\mathcal{S}'}) = \mathrm{Var}(Y_x \mid Y_{\mathcal{S}}) - \mathrm{Cov}(Y_x, Y_{\mathcal{T}} \mid Y_{\mathcal{S}})\, \mathrm{Var}(Y_{\mathcal{T}} \mid Y_{\mathcal{S}})^{-1}\, \mathrm{Cov}(Y_{\mathcal{T}}, Y_x \mid Y_{\mathcal{S}}) \le \mathrm{Var}(Y_x \mid Y_{\mathcal{S}}),
\end{equation}
where $\mathrm{Var}(Y_{\mathcal{T}} \mid Y_{\mathcal{S}}) \succeq \Lambda_{\mathcal{T}} \succ 0$ is invertible because the noise variances are strictly positive. Hence $\Delta F(x \mid \mathcal{S}) \ge \Delta F(x \mid \mathcal{S}')$.

Equivalently, $F$ is the information gain about a fixed set of latent variables $f_{\mathcal{C}}$ from observations $\{Y_x\}$ that are conditionally independent given those variables, a setting in which information gain is known to be monotone submodular \citep{krause2005near}; heteroscedastic noise only changes $\Lambda_{\mathcal{S}}$ and affects neither property.
\end{proof}

\vspace{1em}
\begin{proof}[Proof of Theorem \ref{thm:matroid_ideal}]
Let $\mathcal{I} = \{ \mathcal{T} \subseteq \mathcal{C} : |\mathcal{T} \cap C_i| \le 1 \,\, \forall i=1,\dots,K, \; |\mathcal{T}| \le B \}$ denote the independent sets of the partition matroid truncated to rank $B$, which is a matroid. Algorithm~\ref{alg:greedy} is the greedy algorithm for maximizing $F$ over $\mathcal{I}$: at each step it adds, among all elements whose inclusion keeps the current set independent (those in clusters not yet used), the one with the largest marginal gain $\Delta F(x \mid \mathcal{S}) = \frac{1}{2}\log\delta_x$, and it stops when no independent extension exists, which happens after $B$ steps because $B \le K$.

By Lemma~\ref{lemma:submodularity}, $F$ is normalized, monotone non-decreasing, and submodular. \citet{fisher1978analysis} proved that for any such function and any single matroid constraint, the greedy algorithm returns a set $\mathcal{S}_{greedy}$ with
\begin{equation}
    F(\mathcal{S}_{greedy}) \ge \tfrac{1}{2} \max_{\mathcal{S} \in \mathcal{I}} F(\mathcal{S}) = \tfrac{1}{2} F(\mathcal{S}^*).
\end{equation}
Because $F$ is monotone, the maximum over $\mathcal{I}$ is attained at a basis, so $\mathcal{S}^*$ may be taken to have exactly $B$ queries in distinct clusters; no assumption on $|\mathcal{S}^*|$ is needed for the bound.
\end{proof}

\begin{proof}[Proof of Lemma \ref{lemma:variance_mismatch}]
\emph{Spectral representation.} Let $\mathcal{K}$ be a stationary kernel with spectral density $\hat{\mathcal{K}}$. For any weight vector $w \in \mathbb{R}^{|\mathcal{S}|}$, Bochner's theorem gives the variance of the linear prediction error of $f(x)$ from the noisy observations $Y_{\mathcal{S}}$ as
\begin{equation}
    \mathrm{Var}_{\mathcal{K}}\big(f(x) - w^\top Y_{\mathcal{S}}\big) = \int_{\mathbb{R}^d} |\hat{E}_w(\omega)|^2\, \hat{\mathcal{K}}(\omega)\, d\omega + w^\top \Lambda_{\mathcal{S}} w, \qquad \hat{E}_w(\omega) = 1 - \sum_{i \in \mathcal{S}} w_i e^{i\omega^\top(x_i - x)},
\end{equation}
where the first term is the variance of $f(x) - \sum_i w_i f(x_i)$ and the second is the variance of $\sum_i w_i \epsilon(x_i)$, the two being independent. For jointly Gaussian variables the posterior variance is the minimum prediction-error variance over linear predictors, so
\begin{equation}
    P_{\mathcal{K}}(x \mid \mathcal{S}) = \inf_{w} \left[ \int_{\mathbb{R}^d} |\hat{E}_w(\omega)|^2\, \hat{\mathcal{K}}(\omega)\, d\omega + w^\top \Lambda_{\mathcal{S}} w \right]. \label{eq:spectral_var}
\end{equation}
By Assumption~\ref{assum:spectral}(b), $(1+\|\omega\|^2)^s \hat{\mathcal{K}}^*(\omega)$ is integrable, so $\mathcal{K}^*_s$ is a bona fide stationary kernel and Equation~\eqref{eq:spectral_var} applies to it. Since $\hat{\mathcal{K}}^*_s \ge \hat{\mathcal{K}}^*$ pointwise, the bracket in Equation~\eqref{eq:spectral_var} for $\mathcal{K}^*_s$ dominates that for $\mathcal{K}^*$ at every $w$; infima preserve the order, so $R_s(x \mid \mathcal{S}) = P_{\mathcal{K}^*_s} - P_{\mathcal{K}^*} \ge 0$.

\emph{Part (i).} For the lower bound, fix any $w$. By $\hat{\Phi} \ge c_1 \hat{\mathcal{K}}^*$ and $c_1 \le 1$,
\begin{equation}
    \int |\hat{E}_w|^2 \hat{\Phi} + w^\top \Lambda_{\mathcal{S}} w \;\ge\; c_1 \int |\hat{E}_w|^2 \hat{\mathcal{K}}^* + w^\top \Lambda_{\mathcal{S}} w \;\ge\; c_1 \left[ \int |\hat{E}_w|^2 \hat{\mathcal{K}}^* + w^\top \Lambda_{\mathcal{S}} w \right] \;\ge\; c_1 P_{\mathcal{K}^*}(x \mid \mathcal{S}),
\end{equation}
and taking the infimum over $w$ gives $P_\Phi \ge c_1 P_{\mathcal{K}^*}$. For the upper bound, fix any $w$. By $\hat{\Phi} \le c_2 (1+\|\omega\|^2)^s \hat{\mathcal{K}}^* = c_2 \hat{\mathcal{K}}^*_s$ and $c_2 \ge 1$,
\begin{equation}
    \int |\hat{E}_w|^2 \hat{\Phi} + w^\top \Lambda_{\mathcal{S}} w \;\le\; c_2 \int |\hat{E}_w|^2 \hat{\mathcal{K}}^*_s + w^\top \Lambda_{\mathcal{S}} w \;\le\; c_2 \left[ \int |\hat{E}_w|^2 \hat{\mathcal{K}}^*_s + w^\top \Lambda_{\mathcal{S}} w \right],
\end{equation}
and taking the infimum over $w$ gives $P_\Phi \le c_2 P_{\mathcal{K}^*_s} = c_2 [P_{\mathcal{K}^*} + R_s]$.

\emph{Part (ii).} From part (i), $P_{\mathcal{K}^*} \le P_\Phi / c_1$. Since $c_1 \le 1$,
\begin{equation}
    1 + \frac{P_{\mathcal{K}^*}}{\sigma^2(x)} \;\le\; 1 + \frac{P_\Phi}{c_1 \sigma^2(x)} \;\le\; \frac{1}{c_1}\left(1 + \frac{P_\Phi}{\sigma^2(x)}\right),
\end{equation}
and taking $\frac{1}{2}\log$ of both sides gives $\Delta F_{\mathcal{K}^*}(x \mid \mathcal{S}) \le \Delta F_\Phi(x \mid \mathcal{S}) + \frac{1}{2}\log(1/c_1)$.

\emph{Part (iii).} From part (i), $P_{\mathcal{K}^*} \ge P_\Phi / c_2 - R_s$. Since $c_2 \ge 1$, $\sigma^2(x) \ge \sigma^2(x)/c_2$, and therefore
\begin{equation}
    \sigma^2(x) + P_{\mathcal{K}^*} \;\ge\; \frac{\sigma^2(x) + P_\Phi}{c_2} - R_s.
\end{equation}
Dividing by $\sigma^2(x) + P_\Phi > 0$ and using $R_s \le \bar{R}_s$, $\sigma^2(x) + P_\Phi \ge \sigma^2_{\min}$,
\begin{equation}
    \frac{\sigma^2(x) + P_{\mathcal{K}^*}}{\sigma^2(x) + P_\Phi} \;\ge\; \frac{1}{c_2} - \frac{R_s}{\sigma^2(x) + P_\Phi} \;\ge\; \frac{1}{c_2} - \frac{\bar{R}_s}{\sigma^2_{\min}} \;=\; \frac{1}{c_2}\left(1 - \frac{c_2 \bar{R}_s}{\sigma^2_{\min}}\right),
\end{equation}
which is strictly positive under the hypothesis $c_2 \bar{R}_s < \sigma^2_{\min}$. The left side equals $\exp\{2[\Delta F_{\mathcal{K}^*}(x \mid \mathcal{S}) - \Delta F_\Phi(x \mid \mathcal{S})]\}$, so taking $\frac{1}{2}\log$,
\begin{equation}
    \Delta F_{\mathcal{K}^*}(x \mid \mathcal{S}) - \Delta F_\Phi(x \mid \mathcal{S}) \;\ge\; -\tfrac{1}{2}\log c_2 + \tfrac{1}{2}\log\!\Big(1 - \frac{c_2 \bar{R}_s}{\sigma^2_{\min}}\Big) \;=\; -\tfrac{1}{2}\log c_2 - \varepsilon_s.
\end{equation}
When $s = 0$, $\mathcal{K}^*_s = \mathcal{K}^*$, so $R_s \equiv 0$, $\bar{R}_s = 0$, and $\varepsilon_0 = 0$.

\emph{Consequence.} Since $c_1 \le 1 \le c_2$, both $\frac{1}{2}\log(1/c_1)$ and $\frac{1}{2}\log c_2 + \varepsilon_s$ are at most $\eta = \frac{1}{2}\log(c_2/c_1) + \varepsilon_s$.
\end{proof}
\vspace{1em}

\begin{proof}[Proof of Theorem \ref{thm:robust_matroid}]
Write $F = F_{\mathcal{K}^*}$ and $\Delta F = \Delta F_{\mathcal{K}^*}$ for the objective and marginal gains under the true kernel. Lemma~\ref{lemma:submodularity} applies to any positive-definite kernel matrix, so $F$ is normalized, monotone, and submodular. Let $\mathcal{S}_{greedy} = \{x_1, \dots, x_B\}$ in order of selection, let $\mathcal{S}_i = \{x_1, \dots, x_{i-1}\}$ be the set held before step $i$, and let $\mathcal{C}_i \subseteq \mathcal{C} \setminus \mathcal{S}_i$ be the candidates feasible at step $i$, namely those in clusters not yet used.

\emph{Step 1: Algorithm~\ref{alg:greedy} is an $\eta$-approximate greedy algorithm for $F$.} At step $i$ the algorithm selects $x_i = \arg\max_{x \in \mathcal{C}_i} \Delta F_\Phi(x \mid \mathcal{S}_i)$. For any $x \in \mathcal{C}_i$, applying Lemma~\ref{lemma:variance_mismatch}(iii) to $x_i$, the greedy choice, and Lemma~\ref{lemma:variance_mismatch}(ii) to $x$,
\begin{equation}
    \Delta F(x_i \mid \mathcal{S}_i) \;\ge\; \Delta F_\Phi(x_i \mid \mathcal{S}_i) - \tfrac{1}{2}\log c_2 - \varepsilon_s
    \;\ge\; \Delta F_\Phi(x \mid \mathcal{S}_i) - \tfrac{1}{2}\log c_2 - \varepsilon_s
    \;\ge\; \Delta F(x \mid \mathcal{S}_i) - \eta. \label{eq:approx_greedy}
\end{equation}

\emph{Step 2: approximate greedy under a partition matroid.} Let $O \in \mathcal{I}$ attain $OPT_{\mathcal{K}^*}$. By monotonicity we may take $O$ to be a basis, $B$ elements in distinct clusters, and $\mathcal{S}_{greedy}$ is a basis because the algorithm runs $B$ steps. We index $O = \{o_1, \dots, o_B\}$ so that $o_i \in \mathcal{C}_i$ for every $i$. An element of $O$ that shares its cluster with some $x_i$ receives index $i$; that cluster is unused before step $i$, so the element lies in $\mathcal{C}_i$. The remaining elements of $O$ lie in clusters that the algorithm never uses, hence in every $\mathcal{C}_i$, and they take the remaining indices in any order. Then
\begin{align}
    F(O) &\le F(O \cup \mathcal{S}_{greedy}) && \text{(monotonicity)} \nonumber \\
    &= F(\mathcal{S}_{greedy}) + \sum_{i=1}^B \Delta F\big(o_i \mid \mathcal{S}_{greedy} \cup \{o_1, \dots, o_{i-1}\}\big) && \text{(chain rule)} \nonumber \\
    &\le F(\mathcal{S}_{greedy}) + \sum_{i=1}^B \Delta F(o_i \mid \mathcal{S}_i) && \text{(submodularity)} \nonumber \\
    &\le F(\mathcal{S}_{greedy}) + \sum_{i=1}^B \big[ \Delta F(x_i \mid \mathcal{S}_i) + \eta \big] && \text{(Equation~\eqref{eq:approx_greedy})} \nonumber \\
    &= 2\, F(\mathcal{S}_{greedy}) + B\eta. && \text{(chain rule)}
\end{align}
The third line uses $\mathcal{S}_i \subseteq \mathcal{S}_{greedy} \cup \{o_1, \dots, o_{i-1}\}$, and the fourth applies Equation~\eqref{eq:approx_greedy} with $x = o_i \in \mathcal{C}_i$. If $o_i \in \mathcal{S}_{greedy}$ already, the corresponding term in the second line is zero and the third and fourth lines still hold because $\Delta F \ge 0$. Rearranging gives $F_{\mathcal{K}^*}(\mathcal{S}_{greedy}) \ge \frac{1}{2}\big(OPT_{\mathcal{K}^*} - B\eta\big)$.

This is the argument of \citet{fisher1978analysis} for the exact greedy algorithm, with the exact greedy choice replaced by an $\eta$-approximate one; the only change is the additive $\eta$ per step, which accumulates to $B\eta$ inside the bracket.
\end{proof}
\section{Numerical Diagnostics for Theorem~\ref{thm:robust_matroid}}
\label{app:thm3_diag}

Since every quantity in Theorem~\ref{thm:robust_matroid} can be computed once a true kernel is hypothesized, we report them for the configuration used in Section~\ref{sec:experiments}: $\Phi$ is the Mat\'ern-$1/2$ kernel with unit variance and the median-heuristic length scale $\ell$ fitted in the 30-dimensional projection ($\ell = 1.95$ on the enterprise pool, $\ell = 8.14$ on BIRD), the noise floor is $\sigma^2_{\min} = 0.03$, and $B = 5$. We hypothesize four truths: a smoother Mat\'ern-$\nu^*$ with the same $\ell$ for $\nu^* \in \{3/2, 5/2\}$ (so $s = \nu^* - 1/2$), and a Mat\'ern-$1/2$ with length scale $\ell/2$ or $2\ell$ (so $s = 0$). All kernels use the parameterization of Equation~\eqref{eq:matern_spectral}, so ``same $\ell$'' means the same $\ell$ there, not the same $\alpha$. The constants $c_1, c_2$ are read off the ratio of the two spectral densities on a radial grid; $\mathcal{K}^*_s$ is obtained from its spectral density by a numerical radial inverse Fourier transform; and $\bar{R}_s$ is maximized over the pool and over the empty design, the greedy prefixes of the first batch, and 20 random independent sets, with the posterior conditioned on the current labeled set as in the algorithm. Table~\ref{tab:thm3_diag} gives the results at the first round, where $\bar{R}_s$ is largest, for stationarity dimension $d = 30$ and, for the enterprise pool, also for $d = 18 \approx \hat{d}_I$.

\begin{table}[h]
\centering
\caption{Constants of Theorem~\ref{thm:robust_matroid} on the two pools at round 1 ($n = 5$, $B = 5$). $\bar{R}_s$ and the hypothesis ratio $c_2 \bar{R}_s / \sigma^2_{\min}$ are zero by definition when $s = 0$. $F_\Phi(\mathcal{S}_{greedy})$ is the realized objective of the first greedy batch under the surrogate, in nats; $OPT_{\mathcal{K}^*}$ is of the same order. The bound of Equation~\eqref{eq:robust_bound} is informative only when $B\eta$ is small against it.}
\label{tab:thm3_diag}
\small
\begin{tabular}{llcccccc}
\toprule
\textbf{Pool} ($d$) & \textbf{Hypothesized} $\mathcal{K}^*$ & $s$ & $\tfrac{1}{2}\log(c_2/c_1)$ & $\bar{R}_s$ & $c_2 \bar{R}_s / \sigma^2_{\min}$ & $B\eta$ & $F_\Phi(\mathcal{S}_{greedy})$ \\
\midrule
Enterprise (30) & Mat\'ern-$3/2$, same $\ell$ & 1 & 6.83  & 23.7 & $3.7 \times 10^{8}$  & $\infty$ & 6.09 \\
Enterprise (30) & Mat\'ern-$5/2$, same $\ell$ & 2 & 9.59  & 581  & $1.7 \times 10^{12}$ & $\infty$ & 6.09 \\
Enterprise (30) & Mat\'ern-$1/2$, $\ell/2$ or $2\ell$ & 0 & 10.74 & 0 & 0 & 53.7 & 6.09 \\
Enterprise (18) & Mat\'ern-$3/2$, same $\ell$ & 1 & 3.77  & 14.2 & $4.9 \times 10^{5}$  & $\infty$ & 6.09 \\
Enterprise (18) & Mat\'ern-$5/2$, same $\ell$ & 2 & 5.21  & 224  & $1.1 \times 10^{8}$  & $\infty$ & 6.09 \\
Enterprise (18) & Mat\'ern-$1/2$, $\ell/2$ or $2\ell$ & 0 & 6.59  & 0 & 0 & 32.9 & 6.09 \\
BIRD (30)       & Mat\'ern-$3/2$, same $\ell$ & 1 & 6.83  & 1.38 & $2.1 \times 10^{7}$  & $\infty$ & 8.52 \\
BIRD (30)       & Mat\'ern-$5/2$, same $\ell$ & 2 & 9.59  & 3.35 & $1.0 \times 10^{10}$ & $\infty$ & 8.52 \\
BIRD (30)       & Mat\'ern-$1/2$, $\ell/2$ or $2\ell$ & 0 & 10.74 & 0 & 0 & 53.7 & 8.52 \\
\bottomrule
\end{tabular}
\end{table}

For every hypothesized truth in the table, the right side of Equation~\eqref{eq:robust_bound} is negative. When $s > 0$, the hypothesis $c_2 \bar{R}_s < \sigma^2_{\min}$ of Lemma~\ref{lemma:variance_mismatch}(iii) fails by between six and twelve orders of magnitude, so $\varepsilon_s$ is undefined and no additive bound of this form is available. When $s = 0$, the spectral penalty alone gives $B\eta$ between $33$ and $54$ nats, whereas the per-batch objective is below $9$ nats. Reducing the stationarity dimension to the intrinsic value $d = 18$ roughly halves the constants without changing the conclusion. Two features of the problem are responsible, and neither is a numerical artifact. The ratio $c_2/c_1$ is driven by the ratio of Mat\'ern normalizing constants, which grows exponentially in $d$ for any two kernels that are not identical, so that a factor of two in the length scale alone costs $\tfrac{1}{2}\log(c_2/c_1) \approx 10.7$ at $d = 30$. And $\bar{R}_s$ is essentially the spectral moment $m_s(\mathcal{K}^*) = \mathbb{E}_{\hat{\mathcal{K}}^*}[(1 + \|\omega\|^2)^s] - 1$, which is of order $(d/\ell^2)^s$ unless $\ell \gg \sqrt{d}$. On the enterprise pool this residual decreases only from $23.7$ at $n = 5$ to $21.5$ at $n = 45$, which confirms that the fill-distance regime in which $\bar{R}_s \to 0$ is far from being reached at these budgets. The enterprise pool, with its shorter length scale, is the harder of the two; the BIRD residual is an order of magnitude smaller but still well above the noise floor.

\section*{Acknowledgments}
The author thanks Kangana Beri, Aaditya Shukla, and Jiaxiang Ren at NVIDIA for valuable discussions and feedback that improved this work.

\bibliographystyle{plainnat}
\bibliography{references}
\end{document}